\documentclass{article}

\usepackage{microtype}
\usepackage{graphicx}
\usepackage{subcaption}
\usepackage{booktabs} 
\usepackage{hyperref}

\usepackage{amsmath}
\usepackage{amssymb}
\usepackage{amsthm}
\usepackage{mathtools}
\usepackage{bm}
\usepackage{xcolor}
\usepackage{accents}

\definecolor{blue}{RGB}{0,114,178}
\definecolor{red}{RGB}{213,94,0}
\definecolor{green}{RGB}{0,158,115}
\definecolor{purple}{RGB}{204,121,167}
\definecolor{orange}{RGB}{230,159, 0}
\definecolor{pink}{RGB}{204,121,167}

\newtheorem{theorem}{Theorem}

\newtheorem{lemma}[theorem]{Lemma}

\newcommand\indep{\protect\mathpalette{\protect\independenT}{\perp}}
\def\independenT#1#2{\mathrel{\rlap{$#1#2$}\mkern2mu{#1#2}}}

\newcommand{\E}{\mathop{\mathbb{E}}}
\newcommand{\var}{\mathop{\mathrm{Var}}}
\newcommand{\D}{\mathcal{D}}
\newcommand{\x}{\mathbf{x}}
\newcommand{\ti}{\mathrm{t}}

\newcommand{\y}{\mathrm{y}}

\newcommand{\rt}{\mathrm{r}_{\ti}}
\newcommand{\X}{\mathbf{X}}
\newcommand{\Y}{\mathrm{Y}}
\newcommand{\Yzero}{\Y^0}
\newcommand{\Yone}{\Y^1}
\newcommand{\Yt}{\Y^\ti}
\newcommand{\T}{\mathrm{T}}

\newcommand{\ui}{\mathrm{u}}

\newcommand{\Prob}{P}

\newcommand{\w}{\bm{\omega}}
\newcommand{\W}{\bm{\Omega}}

\newcommand{\rht}{\widehat{\mathrm{r}}_{\ti}}

\newcommand{\muh}{\widehat{\mathrm{\mu}}}

\newcommand{\sigh}{\widehat{\mathrm{\sigma}}}

\newcommand{\et}{e_{\ti}}
\newcommand{\eh}{\widehat{e}}

\newcommand{\ehw}{\eh_{\w}}

\DeclarePairedDelimiter\abs{\lvert}{\rvert}%

\usepackage[accepted]{icml2021}

\icmltitlerunning{Quantifying Ignorance in Individual-Level Causal-Effect Estimates under Hidden Confounding}

\begin{document}

\twocolumn[
\icmltitle{Quantifying Ignorance in Individual-Level Causal-Effect Estimates under Hidden Confounding}



\icmlsetsymbol{equal}{*}

\begin{icmlauthorlist}
    \icmlauthor{Andrew Jesson}{oat}
    \icmlauthor{Sören Mindermann}{oat}
    \icmlauthor{Yarin Gal}{oat}
    \icmlauthor{Uri Shalit}{tech}
\end{icmlauthorlist}

\icmlaffiliation{oat}{OAMTL, University of Oxford}
\icmlaffiliation{tech}{Machine Learning and Causal Inference in Healthcare Lab, Technion -- Israel Institute of Technology}

\icmlcorrespondingauthor{Andrew Jesson}{andrew.jesson@cs.ox.ac.uk}

\icmlkeywords{Causal Inference, Bayesian Deep Learning, Conditional Average Treatment Effect, Heterogeneous Treatment Effect, CATE, BDL, Hidden Confounding, Uncertainty, Epistemic Uncertainty, Sensitivity Analysis}

\vskip 0.3in
]



\printAffiliationsAndNotice{}  

\begin{abstract}
We study the problem of learning conditional average treatment effects (CATE) from high-dimensional, observational data with unobserved confounders. 
Unobserved confounders introduce ignorance---a level of unidentifiability---about an individual's response to treatment by inducing bias in CATE estimates.
We present a new parametric interval estimator suited for high-dimensional data, that estimates a range of possible CATE values when given a predefined bound on the level of hidden confounding.
Further, previous interval estimators do not account for ignorance about the CATE associated with samples that may be underrepresented in the original study, or samples that violate the overlap assumption. 
Our interval estimator also incorporates model uncertainty so that practitioners can be made aware of such out-of-distribution data.
We prove that our estimator converges to tight bounds on CATE when there may be unobserved confounding and assess it using semi-synthetic, high-dimensional datasets.
\end{abstract}

\section{Introduction}

How will a patient's health be affected by taking a given medication \cite{Criado-Perez2020invisible}? 
How will a job seeker's employment be affected by participating in a training program? 
How will a user's question be answered by a search recommendation \citep{noble2018algorithms}?
Making effective personalized recommendations depends on being able to answer such questions.
Answering such questions requires knowledge about the causal effect that a treatment or intervention (medication, training program, search result) has on a person.
And knowing the effect of the treatment requires knowledge about the individual.

Randomized controlled trials (RCTs) are the gold standard for discovering population-level causal effects of such treatments. 
However, in many cases, RCTs are prohibitively expensive or unethical. 
For example, researchers cannot randomly prescribe smoking to assess health risks. 
Observational data, often with larger sample sizes, lower costs, and more relevance to the target population, offer an alternative way to learn about individual-level causal effects. 
The price paid for using observational data, however, is lower certainty in the estimated causal effects. 

When there is sufficient knowledge about both the population and the individual, inferring the individual's response to treatment is possible, and corresponding recommendations can be made with relative certainty. 
A widely used quantity expressing an individual's response to treatment is the Conditional Average Treatment Effect (CATE), which is defined in the next section.

\begin{figure*}[ht!]
    \centering
    \begin{subfigure}{0.24\textwidth}
        \centering
        \includegraphics[width=\linewidth]{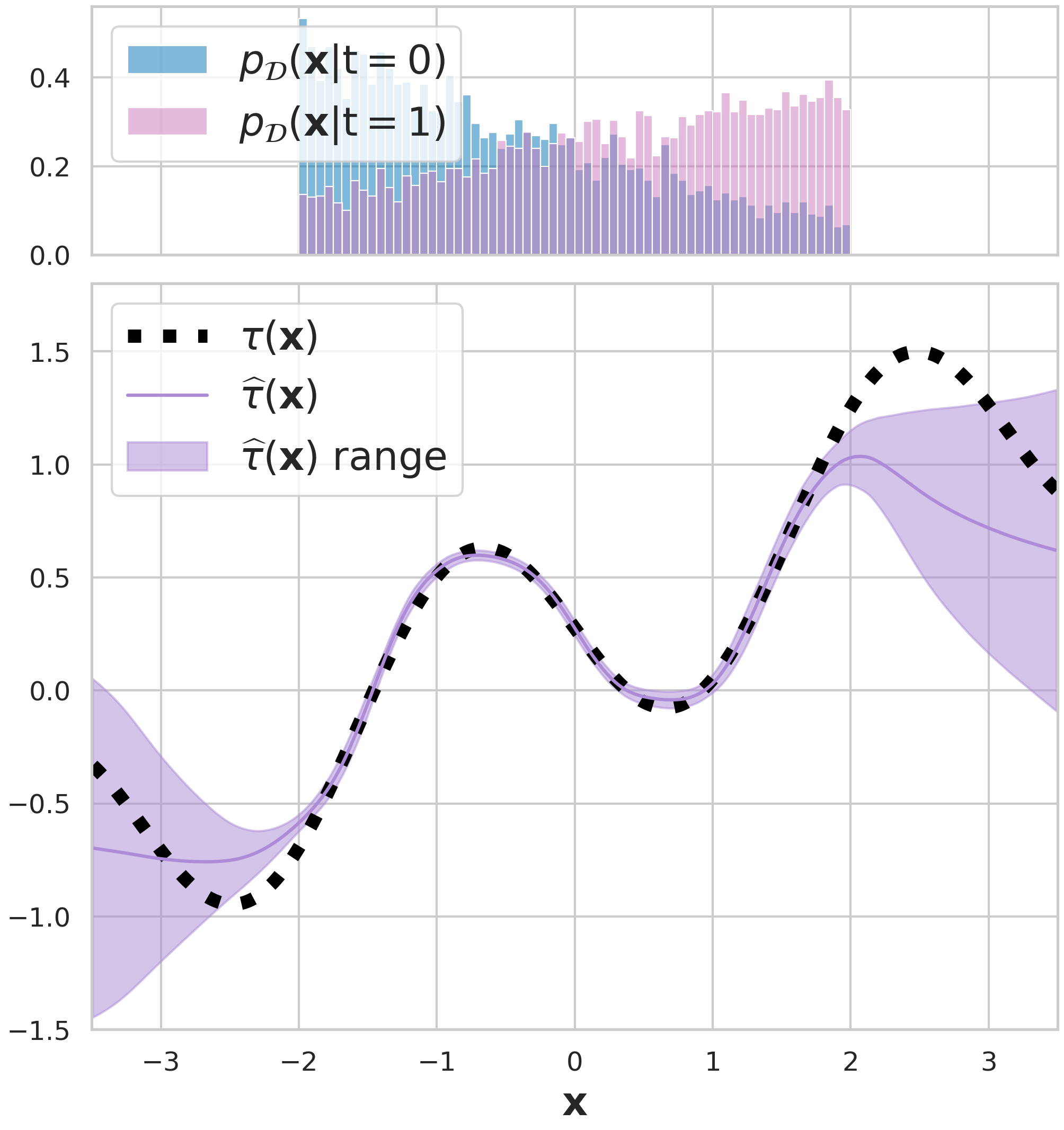}
        \caption{similarity}
        \label{fig:representation}
    \end{subfigure}
    \begin{subfigure}{0.24\textwidth}
        \centering
        \includegraphics[width=\linewidth]{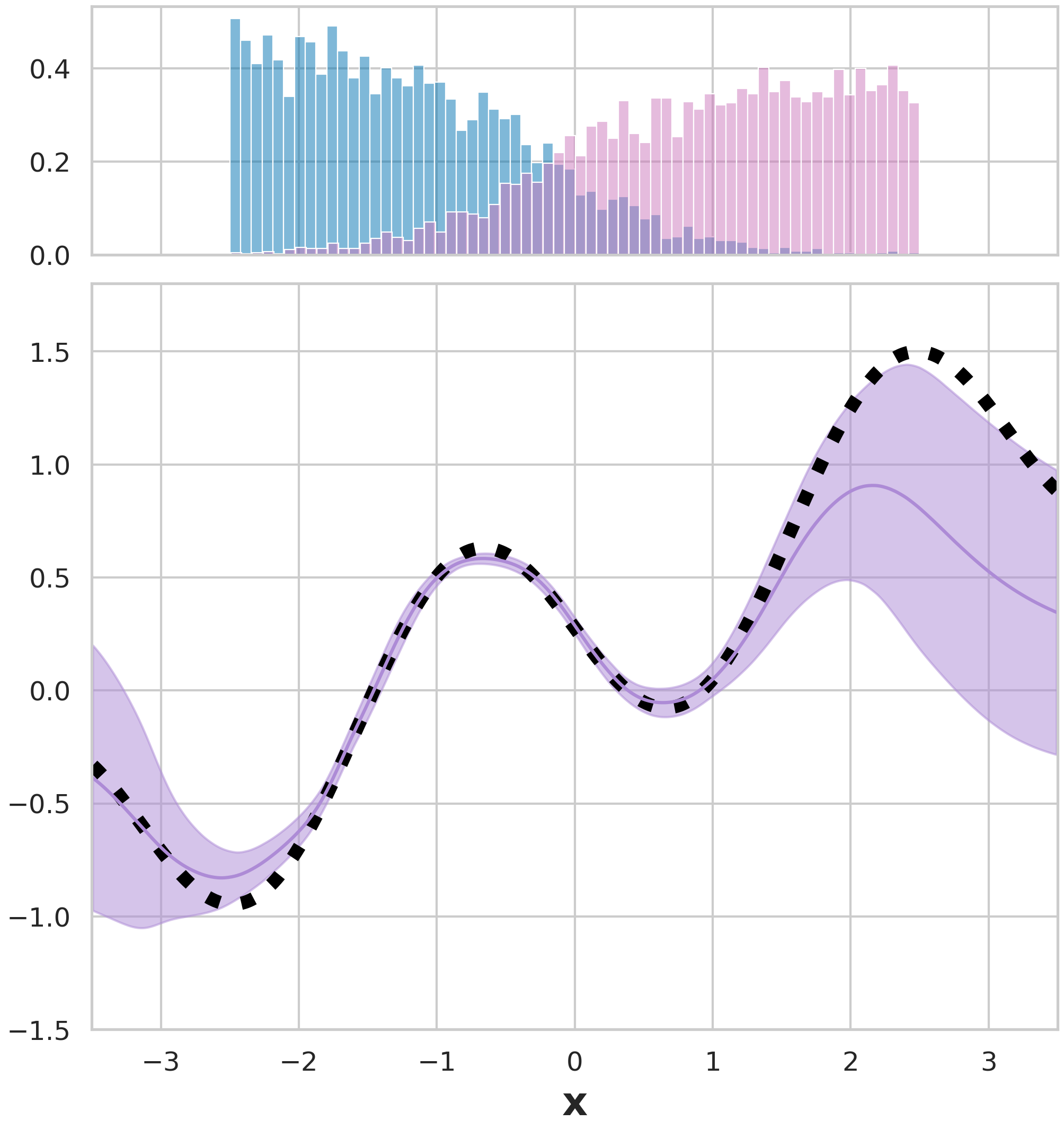}
        \caption{overlap}
        \label{fig:overlap}
    \end{subfigure}
    \begin{subfigure}{0.24\textwidth}
        \centering
        \includegraphics[width=\linewidth]{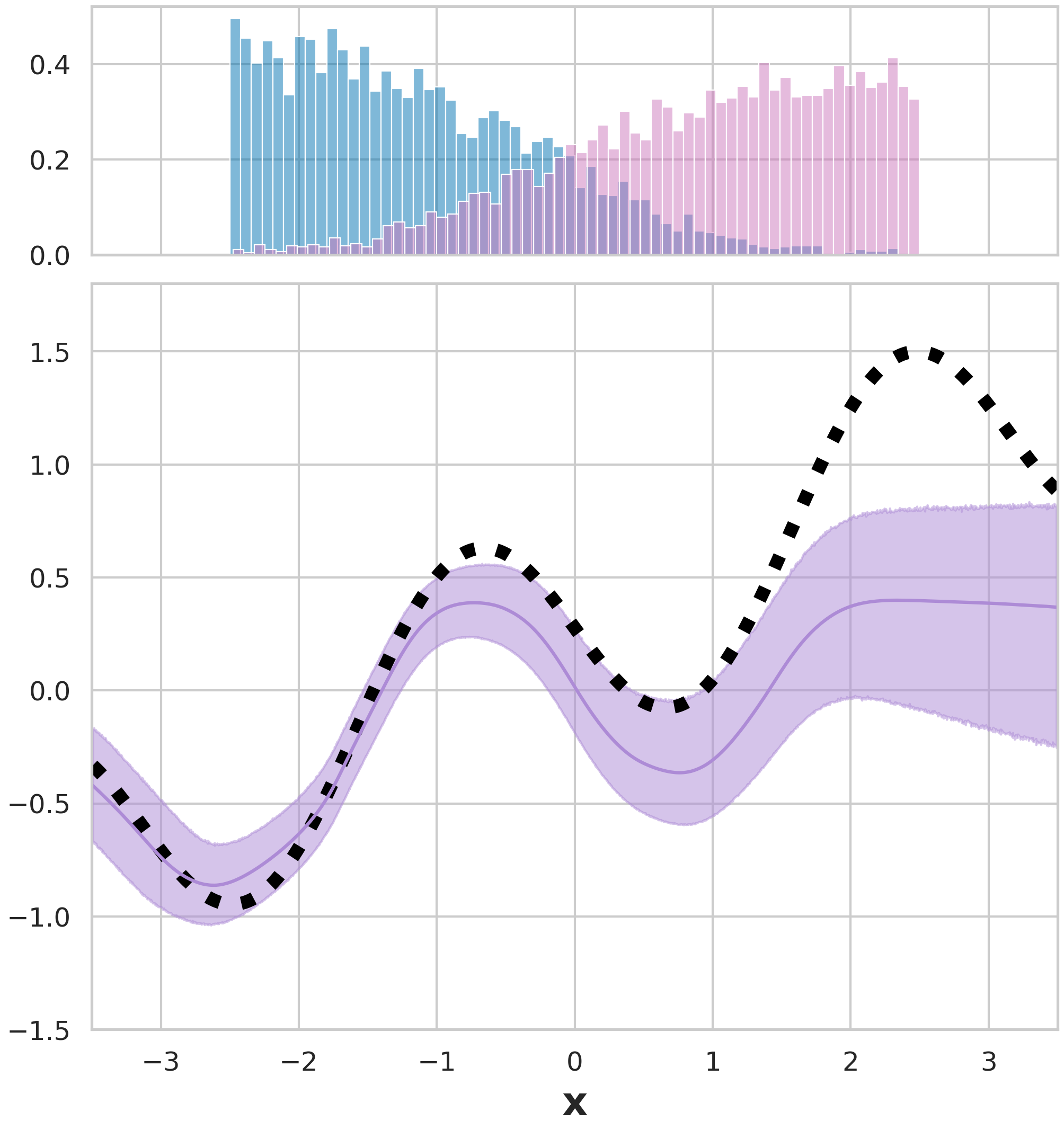}
        \caption{context}
        \label{fig:context}
    \end{subfigure}
    \begin{subfigure}{0.24\textwidth}
        \centering
        \includegraphics[width=\linewidth]{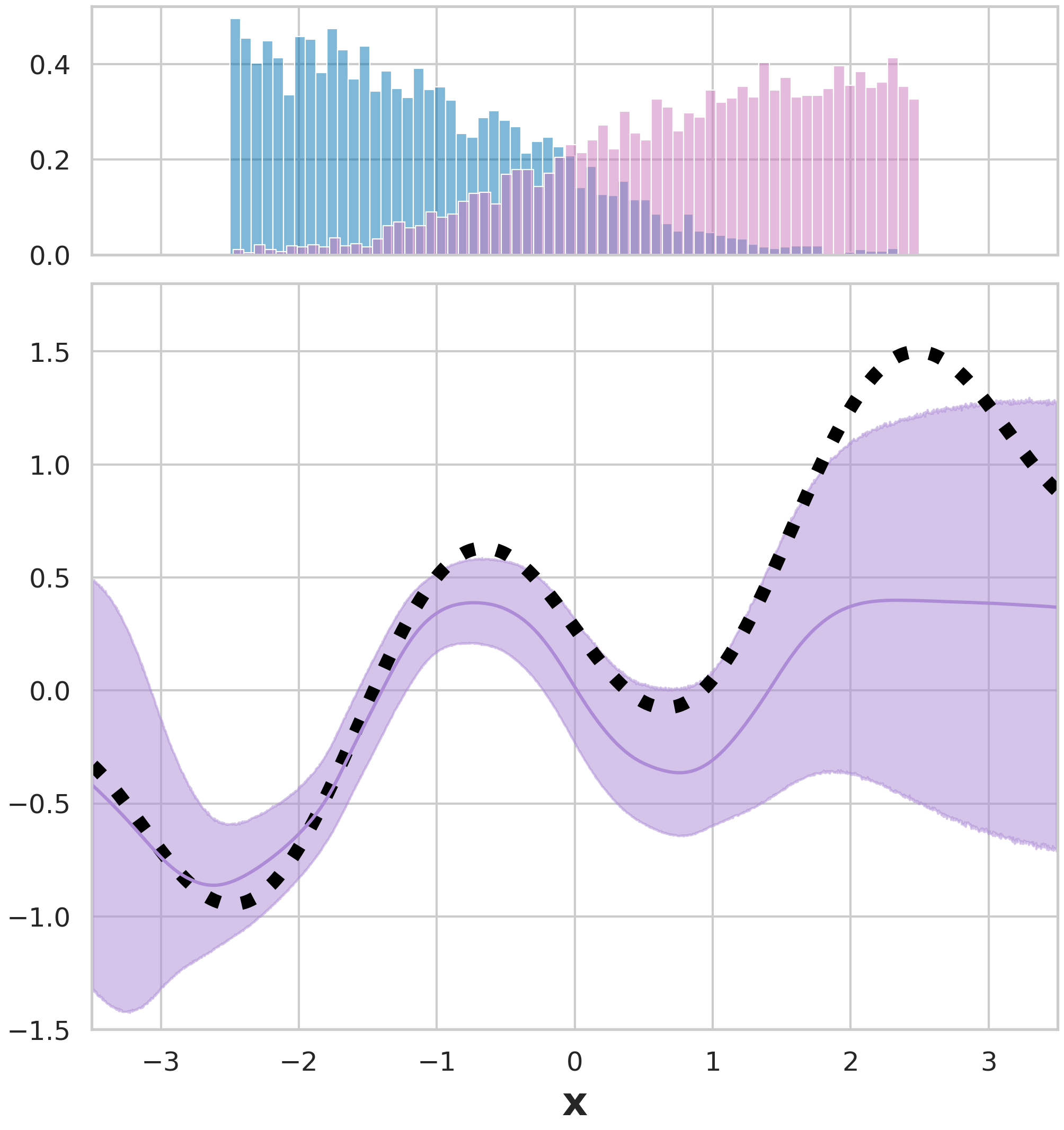}
        \caption{unified}
        \label{fig:unified}
    \end{subfigure}
    \caption{The purple shaded areas in the lower panes depict regions of ignorance about a unit's response to treatment.
    The training data density for the untreated and treated groups are shown in the upper panes.
    (\ref{fig:representation}) For ignorance due to \textbf{insufficient similarity}, the region should get wider as the distance between $\x$ and the training data increases. 
    (\ref{fig:overlap}) For ignorance due to \textbf{insufficient overlap} the region should get wider as $P(\T=0 \mid \x)$ or $P(\T=1 \mid \x) \to 1$. 
    (\ref{fig:context}) Ignorance due to \textbf{insufficient context} results in an arbitrarily biased CATE estimator $\widehat{\tau}(\x)$, hence the discrepancy between the blue solid line and the black dashed line. Therefore, the ignorance region should include the true CATE $\tau(\x)$ on the training data manifold where overlap is satisfied.
    (\ref{fig:unified}) All sources of ignorance jointly.}
    \label{fig:thebigone}
\end{figure*}

There are, however, many reasons why we would not know enough about someone to make an informed recommendation. 
For example, there may be \textbf{insufficient similarity}: when an individual is unrepresented in the study population, which can be the case if the data comes from a small study or just one hospital. 
There may also be \textbf{insufficient overlap} (ubiquitous, especially for high-dimensional data \citep{d2020overlap}): when an individual lacks representation in either the treatment or control group, which can be the case if there are socio-economic barriers to accessing treatment.
Finally, there may be \textbf{insufficient context}: when there are unobserved factors (confounders) that influence both an individual's odds of receiving treatment, as well as their outcome.

When confronted by such ignorance about a person's response to treatment, recommending treatments based on a model's point estimate of the CATE can be dangerous - doubly so in high-stakes domains such as health care. 
Instead, it may be preferable to \textit{defer} the recommendation when the CATE estimate is uncertain: this might entail consulting with a domain expert, using a safe default treatment, collecting additional data on subjects similar to the one in question, or broadening the context of the study by incorporating additional confounding covariates.

In this paper, we provide a measure of ignorance that unifies all three of the above sources  (similarity, overlap, and context), which is expressed as an interval of possible CATE values for each individual.
The width of the interval increases as the assumptions underlying each source are challenged more severely.
When CATE estimates are used to recommend treatment (e.g. ``treat if and only if CATE is positive''), we envision the ignorance interval as being used to \emph{defer} the decision: it might be better not to give a recommendation at all rather than give a highly uncertain one.

We take Bayesian deep learning as a starting point: such methods provide multiple functions to explain observed data (illustrated in Figure \ref{fig:functions}), functions that tend to agree with one another for well-represented data but disagree with one another where data is under or unrepresented. 
Thus, Bayesian methods can be used in quantifying ignorance due to \textbf{insufficient similarity and overlap} by measuring \emph{epistemic uncertainty} (the disagreement between functional predictions of the outcome), which has been used in the context of CATE estimates by \citet{jesson2020identifying}. 

That leaves us with ignorance due to \textbf{insufficient context}, also known as unobserved confounding.
Unobserved confounding manifests as unexplained variance in the estimates of both the outcome and the individual's propensity for treatment and induces a bias in the estimates of causal effects.
Standard Bayesian methods account for unexplained variance in the outcome, known as \emph{aleatoric uncertainty}; however, without further assumptions, it is in general impossible to identify which part of this uncertainty is due to confounding \citep{pearl2009causal}.
Therefore, we turn to causal sensitivity analysis to quantify the ignorance in causal-effect estimates due to the bias induced by hidden confounding.
Causal sensitivity analysis includes a diverse family of frameworks, whose common goal is to give bounds on the treatment-effect under the assumption of some ``level'' of unobserved confounding, either at a population level \citep{rosenbaum1983Assessing,robins2000sensitivity,imbens2003Sensitivity,rosenbaum2014sensitivity,dorie2016flexible,franks2019flexible,veitch2020sense} or at the level of individuals \citep{yadlowsky2018bounds, pmlr-v89-kallus19a}.

Specifically, we build on recent work by \citet{pmlr-v89-kallus19a} and introduce a novel method that can scale to large-sample, high-dimensional data, and convey information about \emph{all three sources of ignorance} mentioned above. 
In section \ref{sec:interval_estimator} we present a new functional interval estimator that predicts a range of possible CATE values when given a bound on the influence of hidden confounding. 
We prove that our estimator converges to tight bounds on CATE for a given bound on hidden confounding. 
In section \ref{sec:uncertainty} we present a CATE interval estimator integrating all sources of uncertainty mentioned above. 
In section \ref{sec:experiments} we demonstrate that our new method scales to high-dimensional data by evaluating it on existing benchmarks and introducing a new high-dimensional dataset.

\section{Sources of Ignorance in Causal Inference}
\label{sec:background}

In this section we formalize the idea of being ignorant about an individual and their response to treatment by framing it as a violation of one or more of the requisite assumptions needed to identify treatment-effects.

The individual's response to treatment is formally known as the \emph{individual treatment effect} or ITE. 
The ITE of a binary treatment $\T \in \{0, 1\}$ on an individual $i$ is the difference in potential outcomes $\Yone_i - \Yzero_i$. 
The potential outcome $\Yone_i$ describes the outcome were the individual $i$ treated, whereas the potential outcome $\Yzero_i$ describes the outcome were they not treated. 
The ITE is a fundamentally unobservable quantity since it is only possible to measure one potential outcome for a given individual. 
However, when individuals are described by a set of covariates $\X \in \mathcal{X} \subseteq \mathbb{R}^d$, then we can model the Conditional Average Treatment Effect (CATE) \citep{abrevaya2015cate},
$\tau(\x) = \E[\Yone - \Yzero \mid \X = \x] 
        = \E[\Yone \mid \X = \x] - \E[\Yzero \mid \X = \x]$,
which is the expected difference in potential outcomes over units (possibly individuals) who share the same measured covariates $\X = \x$.

The estimation of $\tau(\x)$ relies on an observational dataset $\D = \{(\x_i, \ti_i, \y_i): i = 1, \dots n\}$. 
From such data, the expected potential outcome $\E[\Yt \mid \X = \x]$ is identifiable as the conditional expectation over observed outcomes $\mu_{\ti}(\x) \equiv \E[\Y | \T = \ti, \X = \x]$ \citep{rubin1974estimating}  under the following assumptions:
\begin{enumerate}
    \item $(\x_i, \ti_i, \y_i)$ are i.i.d. draws from the same population $P_{\D}(\X, \T, \Yzero, \Yone)$. \label{a:representation}
    \item Overlap (Positivity): \\ $e_{\ti}(\x)  \equiv \Prob(\T = \ti \mid \X =\x) > 0: \ti \in \{0, 1\}$. \label{a:overlap}
    \item Unconfoundedness (Exchaneability, Sufficiency, Exogeneity): $\{(\Yzero, \Yone) \indep \T\} \mid \X$. \label{a:unconfoundedness}
\end{enumerate}

A further assumption which is not our focus here is the stable unit treatment value assumption which briefly stated means that each unit's observed outcome corresponds exactly and only to its treatment assignment. 
That is, for an individual $i$ we observe the outcome $y_i = t_i \Yone_i + (1-t_i)\Yzero_i$.
When these assumptions hold, the CATE for individuals sharing the same measured covariates $\X = \x$ is given by
\begin{equation}
    \tau(\x) = \mu_{1}(\x) - \mu_{0}(\x).
    \label{eq:cate_hat}
\end{equation}

In practice, an estimator $\widehat{\tau}(\x)$ for $\tau(\x)$ is learned from a finite dataset, and ignorance about an individual's response to treatment is due to both observational data being finite and possible violations of the above assumptions.

First, the dataset $\D$ is a finite sample from $P_{\D}(\X, \T, \Yzero, \Yone)$ of size $n$, so there is \textbf{limited similarity} -- for a test point $\x^*$ there might not be any similar train points $\x$. Furthermore, test samples might come from a different marginal distribution $P_{\D'}(\X)$ than the one the training dataset is drawn from, i.e. covariate shift, a scenario which violates Assumption \ref{a:representation}.
Figure \ref{fig:representation} illustrates such violations of Assumption \ref{a:representation}. 
The range of $\widehat{\tau}(\x)$ (purple shaded areas) should be tight around values of $\x$ that are observed in $\D$ and get wider for individuals described by $\x$ that are not.

Second, the treatment assignment may be such that for units described by covariates $\X = \x$, the observed treatment indicator $\T$ is all $0$ or all $1$, so there is \textbf{limited overlap} \citep{d2020overlap}. For example, a given test point $\x^*$ may have similar points in the train set with treatment assignment $T=0$ but none with $T=1$. Therefore, we cannot accurately estimate $\x^*$'s response under $T=1$. Such violations of the overlap assumption are especially common for high-dimensional covariates which likely contain ample information to predict the treatment (Assumption \ref{a:overlap}).
Figure \ref{fig:overlap} illustrates such violations of the \hyperref[a:overlap]{overlap} assumption. 
Here, overlap is not satisfied at the left and right edges of the data.
Therefore, the uncertainty for $\widehat{\tau}(\x)$ should be tight around values of $\x$ for which there are both treated and untreated examples (darker area in top pane, $-2 \leq \x \leq 1.5$) and get wider around values of $\x$ where there are only either treated $(P(\T=1 \mid \x) \to 1 \colon \x > 1.5)$ or untreated examples $(P(\T=0 \mid \x) \to 1 \colon \x < 2)$.

Third, there is \textbf{limited context} about the individual ($\mathcal{X}$ is only $d$-dimensional). For a point $\x^*$ we might not have enough context to correctly estimate its true response under one or both treatments $T$. This is especially important if treatment in the train set was assigned based on an unobserved factor which also affects the outcome $Y$, which is a violation of Assumption \ref{a:unconfoundedness}. 
Figure \ref{fig:context} illustrates such violations of the \hyperref[a:unconfoundedness]{unconfoundedness} assumption.
Such violations result in $\widehat{\tau}(\x)$ (blue solid line) being a biased estimator of the true CATE (black dotted line). 
The bias is induced here by having the probability of treatment and the outcome be affected by a confounding variable $\ui$, which is not included in the set of covariates $\x$ given to the estimator $\widehat{\tau}(\x)$.

A \emph{unified} measure of uncertainty would correspond to the width of the range of CATE values that accounts for all of the above sources of ignorance in the estimate of $\widehat{\tau}(\x)$, for all values of $\x$, as illustrated in Figure \ref{fig:unified}.

\section{Proposed Method}
We first introduce the ideas which are needed to develop our approach: how to evaluate epistemic uncertainty for CATE using Bayesian deep learning \citep{jesson2020identifying}, and a method for expressing violations of \hyperref[a:unconfoundedness]{unconfoundedness} assumption \citep{pmlr-v89-kallus19a} in the context of CATE estimation. We then develop our novel proposed estimator. 
\subsection{Preliminaries}
\subsubsection{Quantifying Ignorance due to Insufficient Similarity and Overlap}

The expectations in Equation \ref{eq:cate_hat} are typically expressed using parametric \citep{robins2000marginal,tian2014simple,shalit2017estimating} or non-parametric models \citep{hill2011bayesian,xie2012estimating,alaa2017bayesian,gao2020minimax}. 
Parametric models assume predictions are generated from $p(\Y \mid \x, \ti, \w)$, the conditional distribution over outcomes $\Y$ given covariates $\x$, treatment $\ti$, and parameters $\w \in \mathcal{W}$.
A common choice for continuous $\Y$ is a Gaussian distribution with density, 
\begin{equation}
    f(\y \mid \x, \ti, \w) = \mathcal{N}\left(\y \mid \widehat{\mu}_{\ti}(\x; \w),  \sigma_{\ti}^2(\x; \w) \right) ,
    \label{eq:gaussian}
\end{equation}
which assumes that $\y$ is given by a deterministic function $\widehat{\mu}_{\ti}(\x; \w)$ with additive Gaussian noise scaled by $\sigma_{\ti}(\x; \w)$. 
For large, high dimensional datasets, neural networks yield suitable functional estimators $\widehat{\mu}_{\ti}(\x; \w)$ and $\widehat{\sigma}_{\ti}(\x; \w)$.
The mean function is then used to define a parametric CATE estimator,
$\widehat{\tau}(\x; \w) = \widehat{\mu}_{1}(\x; \w) - \widehat{\mu}_{0}(\x; \w)$.

Standard neural network optimization often seeks a single set of parameters $\w_{\mathrm{ML}}$ that maximize the likelihood of the observed data $\D$ under the model.
Therefore, it yields one prediction for novel observations $\x^*$, even when an $\x^*$ lies outside of those observed in $\D$, and so there is no way to discern whether $\x^*$ is in-distribution or out-of-distribution.

\begin{figure}[ht]
    \centering
    \includegraphics[width=0.95\linewidth]{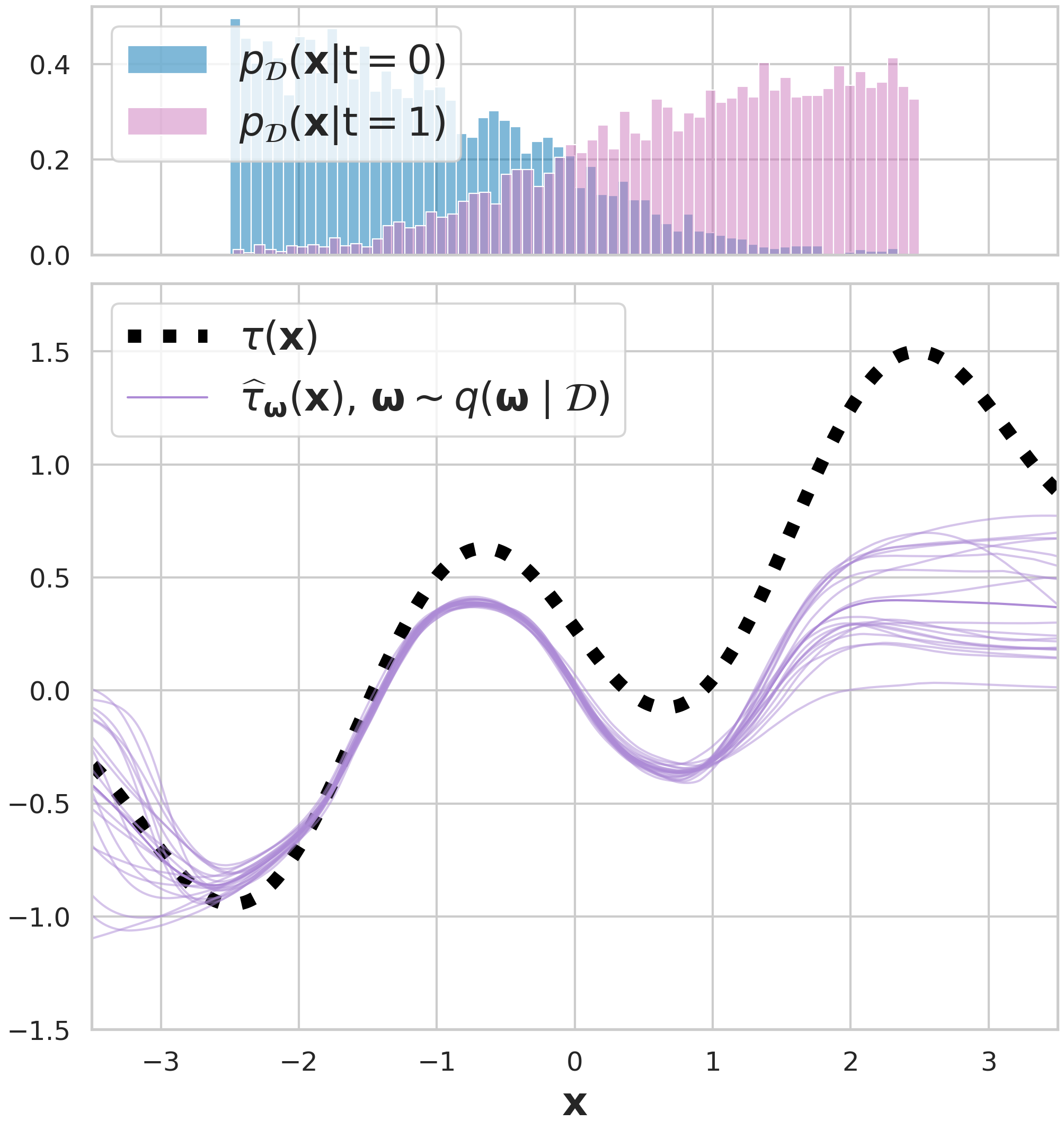}
    \vspace{-4mm}
    \caption{Samples from posterior over functions agree on the training data, but disagree off the training support. However the disagreement does not account for the bias induced by hidden confounding, hence the discrepancy between the purple samples from the model, and the true CATE $\tau(x)$ in the black dashed line.}
    \label{fig:functions}
\end{figure}

Bayesian Deep Learning (BDL), instead aims to generate samples from the posterior distribution of the parameters given the observed data $p(\W \mid \D)$, \emph{e.g.} from a variational approximation of the posterior $q(\W \mid \D)$ \citep{mackay1992practical, hinton1993keeping, barber1998ensemble, gal2016dropout}. 
Ideally, each sample $\w \sim q(\W \mid \D)$ induces a unique functional explanation that, given sufficient flexibility of the neural network, should predict $\y \in \D$. 
When these models work as intended, then for points $\x^*$ far away from the training set $\D$ the function values $\widehat{\mu}_{\ti}(\x^*; \w)$ will have high variance, hence by the law of total variance, so will $\widehat{\tau}(\x^*; \w)$. 
Figure \ref{fig:functions} illustrates how induced functions for the CATE $\widehat{\tau}(\x; \w)$ for different samples of $\w$ agree with one another on the training data, but disagree away from the training data. 
Indeed, in recent work \citet{jesson2020identifying} show that with high-dimensional data, BDL methods are effective at quantifying the uncertainty in CATE estimates arising from insufficient similarity and insufficient overlap. 

Non-parametric methods, such as Bayesian Additive Regression Trees (BART) \citep{hill2011bayesian} or Gaussian Processes (GPs) \citep{alaa2017bayesian} are also capable of expressing such uncertainty, but do not always scale well to big or high-dimensional data. 

While existing Bayesian methods are well suited to account for ignorance due to insufficient similarity and overlap, the approaches above were developed under Assumption \ref{a:unconfoundedness} (unconfoundedness) and so  cannot easily account for the bias in $\widehat{\tau}(\x)$ induced by insufficient context (hidden confounding). 
This is also illustrated in Figure \ref{fig:functions}. 
Specifically, note that even though the functions induced by sampled parameters agree with one another close to the training data, they are still biased away from the true CATE function. 
In order to relax Assumption \ref{a:unconfoundedness}, such ignorance must be accounted for by some other means, as we now discuss.

\subsubsection{Quantifying Ignorance due to Insufficient Context}

\begin{figure}[ht]
    \centering
    \includegraphics[width=1.0\linewidth]{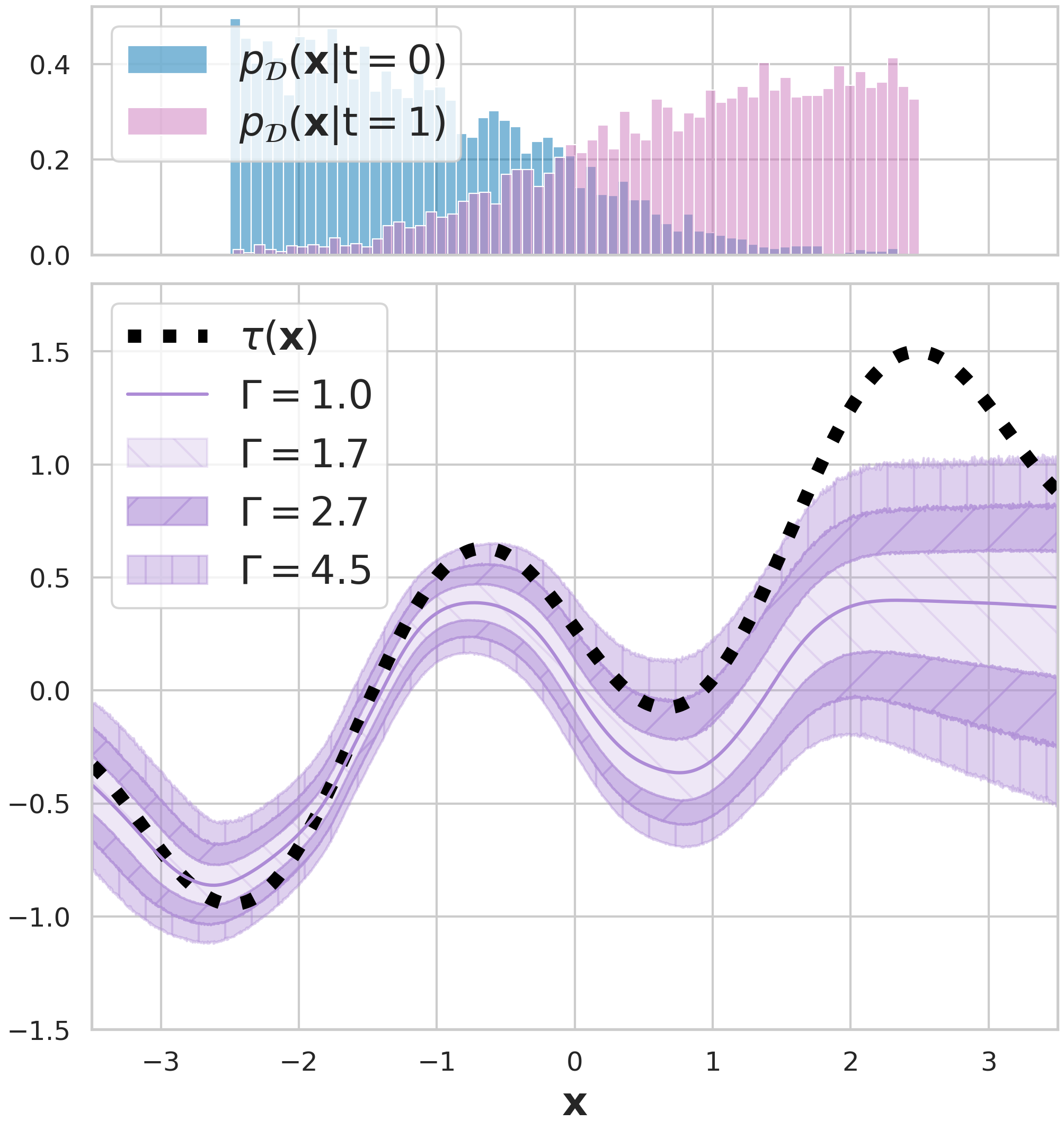}
    \vspace{-4mm}
    \caption{Varying $\Gamma$ for Marginal Sensitivity Model. Ground truth $\Gamma^* = 2.7$. While the bounds follow the true CATE $\tau(x)$ on the support of $p_{\D}(\x)$, they become nonsensical for out-of-distribution data ($\x < -2.5$ and $\x > 2.5$) and when there is a lack of overlap.}
    \label{fig:msm}
\end{figure}

When there is insufficient context, the unconfoundedness assumption $(\Yzero, \Yone) \indep \T \mid \X$ does not necessarily hold. The challenge in this case is to quantitatively express the degree of violation of this conditional independence. \citet{yadlowsky2018bounds} formulate a loss function using the sensitivity model of \citet{RosenbaumPaulR2002Os} for this purpose, but this requires fitting a new model every time the belief in the degree of violation changes. To overcome this limitation, we follow in the footsteps of  \citet{pmlr-v89-kallus19a} who use the Marginal Sensitivity Model (MSM) of \citet{tan2006msm}. 

Let $\et(\x) =  \Prob(\T = \ti \mid \X =\x)$ be the \emph{nominal} propensity score, and $\et(\x, \y) = \Prob(\T = \ti \mid \X =\x ,\Yt = \y)$ be the \emph{complete} propensity score. The complete propensity, being conditioned on the potential outcome, is by construction both unconfounded and unobserved. 
The MSM supposes that the odds of receiving treatment under the complete propensity $\frac{\et(\x, \y) }{ (1 - \et(\x, \y))}$ for individuals described by $\x$ differs from the odds of receiving treatment under the nominal propensity $\frac{\et(\x) }{(1 - \et(\x))}$ by at most a factor of $\Gamma$. That is,
\begin{equation*}
    \Gamma^{-1} \leq \frac{(1 - \et(\x))\et(\x, \y)}{\et(\x)(1 - \et(\x, \y))} \leq \Gamma.
\end{equation*}
As such $\Gamma > 1$ can be interpreted as a degree of supposed hidden confounding, whereas $\Gamma=1$ is equivalent to the \hyperref[a:unconfoundedness]{unconfoundedness} assumption.

In order to incorporate the MSM into a CATE bound, \citet{pmlr-v89-kallus19a} propose using the following factorization for the expectation of the potential outcome $\Yt$:

\begin{equation}
    \begin{split}
        \E[\Yt \mid \X = \x] &= \mu_{\ti}(\textcolor{purple}{w_{\ti}}; \x) \\ 
        &= \frac{\int \y \textcolor{purple}{w_{\ti}(\y \mid \x)} \textcolor{blue}{\et(\x) f(\y \mid \x, \ti)} d\y}{\int \textcolor{purple}{w_{\ti}(\y \mid \x)} \textcolor{blue}{\et(\x) f(\y \mid \x, \ti)} d\y}.
    \end{split}
    \label{eq:unbiased_mu}
\end{equation}

Equation \eqref{eq:unbiased_mu} expresses the unbiased conditional expectation of the potential outcome in terms of the \textcolor{purple}{unidentifiable} inverse complete propensity $\textcolor{purple}{w_{\ti}(\y \mid \x) = 1 / \et(\x, \y)}$ and the \textcolor{blue}{identifiable} nominal propensity \textcolor{blue}{$\et(\x)$} and conditional density \textcolor{blue}{$f(\y \mid \x, \ti)$} of the outcome.

The MSM can then be used to define an ignorance set that includes all possible values of $w_{\ti}(\y \mid \x)$ that would violate unconfoundedness by no more than $\Gamma$, that is
\begin{equation*}
    \mathcal{W}_{\ti}(\x; \Gamma) = \left\{ w_{\ti}: w_{\ti}(\y \mid \x) \in [\alpha_{\ti}(\x; \Gamma), \beta_{\ti}(\x; \Gamma)]  \forall \y \right\},
\end{equation*}
 where $\alpha_{\ti}(\x; \Gamma) = \frac{1}{\Gamma \et(\x)} + 1 - \frac{1}{\Gamma}$, and $\beta_{\ti}(\x; \Gamma) = \frac{\Gamma}{\et(\x)} + 1 - \Gamma$. 
 Given the set $\mathcal{W}_{\ti}(\x; \Gamma)$ expressing bounded violations of unconfoundedness, \citep{pmlr-v89-kallus19a} suggest upper and lower bounds on the CATE as follows: $\overline{\tau}(\x; \Gamma) = \overline{\mu}_{1}(\x; \Gamma) - \underline{\mu}_{0}(\x; \Gamma)$, and 
        $\underline{\tau}(\x;~\Gamma)~=~\underline{\mu}_1(\x;\Gamma)~-~\overline{\mu}_0(\x;\Gamma)$,
where
\begin{subequations}
    \begin{align}
        \underline{\mu}_{\ti}(\x; \Gamma) &= \inf_{w_{\ti} \in \mathcal{W}_{\ti}(\x; \Gamma)} \mu_{\ti}(w_{\ti}; \x). \label{eq:mu_inf} \\
        \overline{\mu}_{\ti}(\x; \Gamma) &= \sup_{w_{\ti} \in \mathcal{W}_{\ti}(\x; \Gamma) } \mu_{\ti}(w_{\ti}; \x) \label{eq:mu_sup} 
    \end{align}
    \label{eq:mu_bounds}
\end{subequations}
Taken together this gives an ignorance interval 
\begin{equation}\label{eq:Tinterval}
    \mathcal{T}(\x; \Gamma) = \left[\underline{\tau}(\x; \Gamma), \overline{\tau}(\x; \Gamma)\right].
\end{equation}
The ignorance interval $\mathcal{T}(\x; \Gamma)$ is completely defined with respect to identifiable estimands. For example, the likelihood in equation \eqref{eq:gaussian} can be used to model the density $f(\y \mid \x, \ti)$ and a parametric model with Bernoulli likelihood,  $P(\T = \ti \mid \x, \w) = \mathrm{Bern}\left(\ti \mid \ehw(\x)\right)$, can be used to model the identifiable nominal propensity for treatment $\et(\x)$.

\citet{pmlr-v89-kallus19a} uses a non-parameteric kernel based method and discrete line search to learn a function that maps $x$ to the identifiable CATE intervals: $\mathcal{T}(\x; \Gamma)$.
Figure \ref{fig:msm} illustrated the bounds given by such a model for given assumptions on $\Gamma$. 

For average treatment effects, there are two approaches for interpreting the bounds on $\tau(\x)$ \citep{tan2006msm}. 
One approach seeks the smallest value $\Gamma_s$ such that the interval $[\underline{\tau}(\x; \Gamma_s), \overline{\tau}(\x; \Gamma_s)]$ crosses 0. This approach then reports that the CATE becomes sensitive to hidden confounding at $\Gamma_s$. 
The other approach sets a cutoff $\Gamma_c$ and examines how the CATE changes for plausible $\Gamma$ values below $\Gamma_c$.

There are two main limitations of the approach of \citep{pmlr-v89-kallus19a} that this paper seeks to address. First, as is evident in the regions of $\x$ that lie out of distribution $(\x < -2.5 \quad \text{or} \quad \x > 2.5)$, the bounds become nonsensical (as expected), and there is no way to identify that a measurement $\x$ is actually out of distribution; more generally, it does not account for sources of ignorance other than unconfoundedness. Second, the method does not scale well computationally to large sample sizes, and does not scale well statistically to high-dimensional datasets as it relies on weighted kernel regression to estimate the outcome. We will now propose a method for incorporating parametric models (including BDL models) instead of the non-parametric one proposed in \citep{pmlr-v89-kallus19a}, thus enabling better scaling to high-dimensional and large-sample setting, while at the same time also accounting for all sources of ignorance.

\subsection{Estimating Bounds on $\widehat{\tau}(\x)$ for a Fixed Degree of Hidden Confounding, $\Gamma$} \label{sec:interval_estimator}

We start by developing a parametric interval estimator for $\mathcal{T}(\x; \Gamma)$ as defined in Eq. \eqref{eq:Tinterval}. 
Our parametric estimator for $\E[\Yt \mid \X=\x]$ under hidden confounding is based off of the following equivalent expression for Equation \eqref{eq:unbiased_mu}
\begin{equation*}
    \mu_{\ti}(w_\ti; \x) = \textcolor{blue}{\mu_{\ti}(\x)} + \frac{ \int \rt(\y; \x) \textcolor{purple}{w_\ti(\y \mid \x)} \textcolor{blue}{f(\y \mid \x, \ti)} d\y }{\int \textcolor{purple}{w_\ti(\y \mid \x)} \textcolor{blue}{f(\y \mid \x, \ti)} d\y},
\end{equation*}
where the residual is given by $\rt(\y; \x) = (\y - \textcolor{blue}{\mu_{\ti}(\x)})$ (see Lemma \ref{lem:double} in the Appendix for proof). This expression is still given in terms of both \textcolor{blue}{identifiable} and \textcolor{purple}{unidentifiable} quantities.

Building off the derivation in Lemma 1 of \citet{pmlr-v89-kallus19a}, we can then express the infimum and supremum in \eqref{eq:mu_bounds} as
\begin{subequations}
    \begin{align*}
        \underline{\mu}_{\ti}(\x; \Gamma) &= \inf_{\y^* \in \mathcal{Y}} \textcolor{blue}{\mu_{\ti}(\x)} + \frac{ \int_{-\infty}^{\y^*} \rt(\y; \x) \textcolor{blue}{f(\y \mid \x, \ti)} d\y }{ \textcolor{blue}{\alpha_{\ti}^{\prime}(\x; \Gamma)} + \textcolor{blue}{\mathrm{P}(\Y \leq \y^* \mid \x, \ti)} }, \\
        \overline{\mu}_{\ti}(\x; \Gamma) &= \sup_{\y^* \in \mathcal{Y}} \textcolor{blue}{\mu_{\ti}(\x)} + \frac{ \int_{\y^*}^{\infty} \rt(\y; \x) \textcolor{blue}{f(\y \mid \x, \ti)} d\y }{ \textcolor{blue}{\alpha_{\ti}^{\prime}(\x; \Gamma)} + \textcolor{blue}{\mathrm{P}(\Y > \y^* \mid \x, \ti)} },
    \end{align*}
\end{subequations}
where $\textcolor{blue}{\alpha_{\ti}^{\prime}(\x; \Gamma) = \frac{\alpha_{\ti}(\x; \Gamma)}{\beta_{\ti}(\x; \Gamma) - \alpha_{\ti}(\x; \Gamma)}}$ and $\mathcal{Y}$ is the space of outcomes as before (see Lemma \ref{lem:density} in the Appendix for proof). 
As such, the bounds on $\E[\Yt \mid \X=\x]$ given $\Gamma$ are now completely defined in terms of \textcolor{blue}{identifiable} quantities; 
namely, the nominal propensity for treatment \textcolor{blue}{$\et(\x)$}, the conditional distribution of the outcome \textcolor{blue}{$p(\Y \mid \x, \ti)$}, and its density function \textcolor{blue}{$f(\y \mid \x, \ti)$}, from which $\textcolor{blue}{\alpha_{\ti}^{\prime}(\x; \Gamma)}$, $\textcolor{blue}{\mu_{\ti}(\x)}$, and $\textcolor{blue}{\mathrm{P}(\cdot \mid \x, \ti)}$ are straightforwardly derived.

Where \citet{pmlr-v89-kallus19a} use a kernel-based estimator, we instead model the identifiable $p(\Y \mid \x, \ti)$ directly. 
Thus, a parametric generative model $p(\Y \mid \x, \ti, \w)$ from which to sample $\y$ and evaluate $\widehat{\mu}_{\ti}(\x; \w)$, and a propensity score estimator $\eh_{\ti}(\x; \w)$ to evaluate $\alpha_{\ti}^{\prime}(\x; \Gamma, \w)$ are needed. 
$\alpha_{\ti}^{\prime}(\x; \Gamma, \w)$ is calculated in the same manner as $\alpha_{\ti}^{\prime}(\x; \Gamma)$, where $\et(\x)$ is replaced by $\eh_{\ti}(\x; \w)$ in the terms for $\alpha_{\ti}(\x; \Gamma)$ and $\beta_{\ti}(\x; \Gamma)$.

Hidden confounders induce multi-modal distributions over $\Y$; therefore, we model $p(\Y \mid \x, \ti, \w)$ with a Gaussian Mixture denisty over $J$ mixture components, noting that with a sufficient number of mixture components it can approximate any continuous distribution \citep{titterington1985statistical}. 
Thus, our density function $f(\y \mid \x, \ti, \w)$ is
\begin{equation*}
    \sum_{j=1}^{J} \widehat{\pi}_{\ti}^{j}(\x; \w) \mathcal{N}\left( \y \mid \muh_{\ti}^{j}(\x; \w),  {\sigh_{\ti}^{j2}}(\x; \w) \right),
\end{equation*}
and $\widehat{\mu}_{\ti}(\x; \w) = \sum_{j=1}^{J} \widehat{\pi}_{\ti}^{j}(\x; \w) \muh_{\ti}^{j}(\x; \w)$  \citep{bishop1994mixture}.
We expand on this choice in Appendix \ref{a:implementation_details}.

Given these models, we can now define the parametric interval CATE estimator, $\widehat{\mathcal{T}}(\x; \Gamma, \w) = [\widehat{\underline{\tau}}(\x; \Gamma, \w), \widehat{\overline{\tau}}(\x; \Gamma, \w)]$:
\begin{subequations}
    \begin{align}
    \widehat{\underline{\tau}}(\x; \Gamma, \w) &= \widehat{\underline{\mu}}_{1}(\x; \Gamma, \w) - \widehat{\overline{\mu}}_{0}(\x; \Gamma, \w),  \\
    \widehat{\overline{\tau}}(\x; \Gamma, \w) &= \widehat{\overline{\mu}}_{1}(\x; \Gamma, \w) - \widehat{\underline{\mu}}_{0}(\x; \Gamma, \w),
    \end{align}
    \label{eq:tau_interval}
\end{subequations}
where
\begin{subequations}
    \begin{align}
        \widehat{\underline{\mu}}_{\ti}(\x; \Gamma, \w) &= \inf_{\y^* \in \mathcal{Y}} \widehat{\underline{\lambda}}_{\ti}(\y^*; \x, \Gamma, \w), \\
        \widehat{\overline{\mu}}_{\ti}(\x; \Gamma, \w) &= \sup_{\y^* \in \mathcal{Y}} \widehat{\overline{\lambda}}_{\ti}(\y^*; \x, \Gamma, \w),
    \end{align}
    \label{eq:mu_def}
\end{subequations}
and for $\rht(\y; \x, \w) = \y-\widehat{\mu}_{\ti}(\x; \w)$:
\begin{subequations}
    \begin{align*}
        \widehat{\underline{\lambda}}_{\ti}(&\y^*; \x, \Gamma, \w) = \\
        &\widehat{\mu}_{\ti}(\x; \w) + \frac{ \int_{-\infty}^{\y^*} \rht(\y; \x, \w) f(\y \mid \x, \ti, \w) d\y }{ \alpha_{\ti}^{\prime}(\x; \Gamma, \w) + \int_{-\infty}^{\y^*} f(\y \mid \x, \ti, \w) d\y }, \\
        \widehat{\overline{\lambda}}_{\ti}(&\y^*; \x, \Gamma, \w) = \\
        &\widehat{\mu}_{\ti}(\x; \w) + \frac{ \int_{\y^*}^{\infty} \rht(\y; \x, \w) f(\y \mid \x, \ti, \w) d\y }{ \alpha_{\ti}^{\prime}(\x; \Gamma, \w) + \int_{\y^*}^{\infty} f(\y \mid \x, \ti, \w) d\y }.
    \end{align*}
    \label{eq:lambda}
\end{subequations}

\subsection{Computing the Interval Estimator}
\label{sec:computing_the_interval}
Where \citet{pmlr-v89-kallus19a} define their interval estimator as an optimization problem over $n$ weight variables, where $n$ is the size of the training set, we instead characterize ours as an optimization problem over $m$ samples of $\y$ from the modeled conditional distribution $p(\Y \mid \x, \ti, \w)$.
Because $\widehat{\underline{\lambda}}_{\ti}(\y^*; \x, \Gamma, \w)$ is convex and $\widehat{\overline{\lambda}}_{\ti}(\y^*; \x, \Gamma, \w)$ is concave with increasing $y^*$, we can employ a similar discrete line search as \citet{pmlr-v89-kallus19a} to solve this optimization problem. However, where their search has $\mathcal{O}(n)$ time complexity, ours is independent of the dataset size and has $\mathcal{O}(m)$ time complexity.
$m$ is a user-defined parameter that controls the stability of predicted $\widehat{\underline{\lambda}}_{\ti}(k; \x, \Gamma, \w)$ and $\widehat{\overline{\lambda}}_{\ti}(k; \x, \Gamma, \w)$, defined below. 

This solution uses Monte-Carlo integration to estimate $\widehat{\underline{\lambda}}_{\ti}(\y^*; \x, \Gamma, \w)$ and $\widehat{\overline{\lambda}}_{\ti}(\y^*; \x, \Gamma, \w)$, so as $m$ increases the Monte-Carlo estimates converge to the integral.
One could use other methods to evaluate the integrals, such as Bayesian Quadrature.

The algorithm proceeds by reordering the samples of $\y$ such that $\y_1 \leq \y_2 \leq \dots \y_m$ and defining the following terms for $k \in \{ 1, \dots, m \}$, $\x \in \mathcal{X}$, and $\Gamma \geq 1$:
\begin{subequations}
    \begin{align*}
        \widehat{\underline{\lambda}}_{\ti}(k; \x, \Gamma, \w) &= \widehat{\mu}_{\ti}(\x; \w) + \frac{ \frac{1}{m} \sum_{i=1}^{k} \rht(\y; \x, \w) }{\alpha_{\ti}^{\prime}(\x; \Gamma, \w) + \frac{k}{m} }, \\
        \widehat{\overline{\lambda}}_{\ti}(k; \x, \Gamma, \w) &= \widehat{\mu}_{\ti}(\x; \w) + \frac{ \frac{1}{m} \sum_{i=k + 1}^{m} \rht(\y; \x, \w) }{\alpha_{\ti}^{\prime}(\x; \Gamma, \w) + 1 - \frac{k}{m} }.
    \end{align*}
    \label{eq:lambda_opt}
\end{subequations}
Then, $\widehat{\underline{\mu}}_{\ti}(\x; \Gamma, \w) = \widehat{\underline{\lambda}}_{\ti}(k^{L}; \x, \Gamma, \w)$, and $\widehat{\overline{\mu}}_{\ti}(\x; \Gamma, \w) = \widehat{\overline{\lambda}}_{\ti}(k^{H}; \x, \Gamma, \w)$,
with
\begin{equation*}
    \begin{split}
        k^{L} &= \inf\left\{ \widehat{\underline{\lambda}}_{\ti}(k; \x, \Gamma, \w) \leq \widehat{\underline{\lambda}}_{\ti}(k + 1; \x, \Gamma, \w) \right\}_{k=1}^{m} \\
        k^{H} &= \inf\left\{\widehat{\overline{\lambda}}_{\ti}(k; \x, \Gamma, \w) \geq \widehat{\overline{\lambda}}_{\ti}(k + 1; \x, \Gamma, \w) \right\}_{k=1}^{m}.
    \end{split}
\end{equation*}

\subsection{Tightness of Bounds}
\begin{theorem}
Suppose that
    \renewcommand{\labelenumi}{\roman{enumi}}
    \begin{enumerate}
        \item $n \to \infty$, and $\x \in \D$.
        \item $\Y$ is a bounded random variable.
        \item $f(\y \mid \x, \ti, \w)$ converges in measure to $f(\y \mid \x, \ti)$.
        \item $\eh_{\ti}(\x; \w)$ and $\widehat{\mu}_{\ti}(\x; \w)$ are consistent estimators of $\E[\T=\ti \mid \X=\x]$ and $\E[\Y \mid \X=\x, \T=\ti]$.
        \item $\et(\x, \y)$ is bounded away from 0 and 1 uniformly over $\x \in \mathcal{X}$, $\y \in \mathcal{Y}$, and $\ti \in \{0, 1\}$   (overlap assumption).
    \end{enumerate}
Then, for $\ti \in \{0, 1\}$, $\widehat{\underline{\mu}}_{\ti}(\x; \Gamma, \w) \xrightarrow[]{p} \underline{\mu}_{\ti}(\x; \Gamma) $, $\widehat{\overline{\mu}}_{\ti}(\x; \Gamma, \w) \xrightarrow[]{p} \overline{\mu}_{\ti}(\x; \Gamma) $, which imply that both $\widehat{\underline{\tau}}(\x; \Gamma, \w) \xrightarrow[]{p} \underline{\tau}(\x; \Gamma)$, and $\widehat{\overline{\tau}}(\x; \Gamma, \w) \xrightarrow[]{p} \overline{\tau}(\x; \Gamma)$.
\label{th:convergence}
\end{theorem}

Proof for Theorem \ref{th:convergence} is given in Appendix \ref{sec:estimator_proof}.

\subsection{Model Uncertainty for $\widehat{\tau}(\x)$ Interval Estimates} \label{sec:uncertainty}

In order to account for uncertainty that arises from both a lack of similarity and a lack of overlap, it is necessary to propagate the uncertainty in $\w$ to the estimates of the lower and upper bounds on $\widehat{\tau}(\x)$. 
The CATE bounds above are for a given parameterization $\w \sim q(\W \mid \D)$. By taking the expectation over $\w$, we arrive at

\begin{equation}
    \begin{split}
        \widehat{\underline{\tau}}(\x; \Gamma) &= \E_{\w} [ \widehat{\underline{\tau}}(\x; \Gamma, \w) ] - 2 \cdot \sqrt{\var_{\w} \left[ \widehat{\underline{\tau}}(\x; \Gamma, \w) \right]} \\
        \widehat{\overline{\tau}}(\x; \Gamma) &= \E_{\w} [ \widehat{\overline{\tau}}(\x; \Gamma, \w) ] + 2 \cdot \sqrt{\var_{\w} \left[ \widehat{\overline{\tau}}(\x; \Gamma, \w) \right]}
    \end{split}
    \label{eq:tau_hats}
\end{equation}

\begin{equation*}
    \widehat{\mathcal{T}}(\x; \Gamma) =\left[\widehat{\underline{\tau}}(\x; \Gamma) , \widehat{\overline{\tau}}(\x; \Gamma) \right],
    \label{eq:tau_expected}
\end{equation*}

which we name the predictive interval (using 2 standard deviations as the Bayesian confidence level here).
The expectations and variances in eq. \eqref{eq:tau_hats} can then be evaluated via Monte Carlo integration.

\begin{table*}[ht]
    \centering
    \caption{\textbf{Simulated Data:} Policy risk errors for various policies under data generating processes with different $\Gamma^*$. Average test-set policy risk errors and 95\% confidence intervals over 50 randomly generated datasets are reported. Statistically significant improvements for ``well-specified'' $\Gamma = \Gamma^*$, as determined by a paired t-test (1\% threshold), shown in \textcolor{green}{green}. Policy risk errors are multiplied by 100 for readability.}
    \begin{tabular}{c|ccc|ccc}
        \toprule
        $n=1000$ & \multicolumn{3}{c}{Proposed Method} & \multicolumn{3}{c}{\citet{pmlr-v89-kallus19a}} \\
         $\log{\Gamma^*}$ & $\hat{\pi}(\x; \exp(0.5))$ & $\hat{\pi}(\x; \exp(1.0))$ & $\hat{\pi}(\x; \exp(1.5))$ & $\hat{\pi}(\x; \exp(0.5))$ & $\hat{\pi}(\x; \exp(1.0))$ & $\hat{\pi}(\x; \exp(1.5))$ \\
         \midrule
         0.5 & $\mathbf{0.07\pm0.03}$ & $0.28\pm0.03$ & $0.38\pm0.04$ & $\mathbf{0.10\pm0.04}$ & $0.44\pm0.09$ & $0.99\pm0.38$ \\
         1.0 & $0.71\pm0.20$ & \textcolor{green}{$\mathbf{0.10\pm0.04}$} & $0.31\pm0.03$ & $0.48\pm0.19$ & $\mathbf{0.25\pm0.11}$ & $0.81\pm0.39$ \\
         1.5 & $3.99\pm0.59$ & $0.75\pm0.18$ & \textcolor{green}{$\mathbf{0.13\pm0.04}$} & $3.33\pm0.61$ & $0.52\pm0.19$ & $\mathbf{0.52\pm0.40}$ \\ 
         \bottomrule
    \end{tabular}
    \label{tab:synthetic}
    \vspace{-4mm}
\end{table*}

\section{Experiments} \label{sec:experiments}

In this section we evaluate our methods using synthetic and semi-synthetic datasets. 
To assess our method on high-dimensional data, we introduce a new benchmark dataset, HC-MNIST. 
To illustrate how our uncertainty aware bounds can be used for deferring treatment, we introduce a hidden confounding variant of the IHDP dataset \citep{hill2011bayesian}. 
Details about the data generating processes including dataset links, code links, and validation splitting procedures are given in Appendix \ref{a:datasets}.

The sampling procedure outlined in subsections \ref{sec:computing_the_interval}-\ref{sec:uncertainty} for the estimator in eq. \eqref{eq:tau_hats} requires models for $p(\Y \mid \x, \ti)$ and the nominal propensity $\et(\x)$.
We use a mixture density network for $p(\Y \mid \x, \ti, \w)$ and a standard neural network with categorical likelihood for $\eh_{\ti}(\x; \w)$.
Deep Ensembles \citep{balaji2017ensemble} are used to approximate sampling $\w \sim p(\W \mid \D)$. 
In general, modelling $p(\W \mid \D)$ is a choice to be made by the practitioner, for example, by using Bayesian Neural Networks or simpler Bayesian models for $p(\Y \mid \x, \ti, \w)$.
Details for each experiment, including architectures, hyper-parameter tuning, training procedures, and compute infrastructure are detailed in Appendix \ref{a:implementation_details}.

\subsection{Simulated Data} 
We first consider the one-dimensional example introduced by \citet{pmlr-v89-kallus19a} \ref{a:datasets_simulated}. 
Figure \ref{fig:msm}, generated with $n=10000$ and $\log{\Gamma^*} = 1$, illustrates the nonlinear CATE function of these data. 
This is a useful example because both the CATE and the bias induced by hidden confounding are heterogeneous in $\x$. 
Further, Figure \ref{fig:msm} shows that our estimator, outlined in sections \ref{sec:interval_estimator} and \ref{sec:computing_the_interval}, converges to tight bounds on the CATE interval for varying choices of $\Gamma$, achieving coverage when the assumed $\Gamma$ matches the true value $\Gamma^*$ used to generate the data. 
For this experiment and the next we assume that the outcomes correspond to costs, so that we aim to treat when $\tau(\x) \leq 0$.

For a quantitative evaluation, we use the same minimax-optimal policy as \citet{pmlr-v89-kallus19a}, namely, $\pi^*(\x; \Gamma) = \mathbb{I}(\overline{\tau}(\x; \Gamma) \leq 0) + \pi_0(\x) \mathbb{I}( \underline{\tau}(\x; \Gamma) < 0 < \overline{\tau}(\x; \Gamma) ) $. 
This says that the optimal policy always treats when $\overline{\tau}(\x; \Gamma) \leq 0$ and otherwise reverts to the default policy $\pi_0(\x)$. Setting $\pi_0(\x) = 0$, \emph{do not treat}, our approximation to the optimal policy is given by $\widehat{\overline{\pi}}(\x; \Gamma) = \mathbb{I}(\widehat{\overline{\tau}}(\x; \Gamma) \leq 0)$. 
The risk associated with a given policy is defined as $V(\pi; \tau) = \E[\pi(\x)\Yone + (1 - \pi(\x))\Yzero]$. 
Intuitively, policy risk will be minimized when $\widehat{\overline{\tau}}(\x)$ is aligned exactly with the true CATE $\tau(\x)$, and any deviations between $\widehat{\overline{\tau}}(\x)$ and $\tau(\x)$ will result in a higher policy risk score.
To compare different methods on a finite sample, we report the \emph{Policy Risk Error} as the mean squared error between the risk of the optimal treatment policy $\mathbb{I}(\tau(x) < 0)$, and the policy risk of a given policy $\pi$. 

In Table \ref{tab:synthetic}, we compare the Policy Risk Error of our method to the one proposed by \citet{pmlr-v89-kallus19a}. 
The average and $95\%$ confidence intervals over 50 random realizations of training ($n=1000$), validation ($n=100$), and test ($n=1000$) datasets are reported. 
On the diagonals we assess each policy and method with a ``well-specified'' $\Gamma = \Gamma^*$. 
These results show empirical evidence for the tightness of our interval estimator's bounds, and improved accuracy w.r.t. \citet{pmlr-v89-kallus19a} on this low-dimension problem.

\subsection{HC-MNIST: Hidden Confounding with High-dimensional Data}

For this experiment, we adopt the one-dimensional simulated setting into a high-dimensional setting \ref{a:datasets_hcmnist}. Specifically, we assign to each image of the MNIST dataset \citep{lecun1998mnist} a latent feature $\phi \in [-2,2]$ as follows: all images of the digits $0$ are assigned a $\phi \in [-2,-1.6]$, all images $1$ have $\phi \in [-1.6,-1.2]$, and so on up to the digit $9$. The images of every digit are sorted by brightness and ordered equally within the interval of $\phi$ values assigned to images of that digit. Finally, these one-dimensional hidden values $\phi$ are used as the inputs to the same model of hidden confounding introduced by \citet{pmlr-v89-kallus19a} and used in the simulated data experiments above. We report the results of our method in Table \ref{tab:HC-MNIST}, showing it achieves near optimal policy risk under the true level of hidden confounding. We do not to report results for \citet{pmlr-v89-kallus19a} here as their kernel based method did not scale well to the full dataset size of MNIST, and it did not give sensible results when training only on a subset of the dataset.

\begin{table}[t]
    \centering
    \caption{\textbf{HC-MNIST:} Policy risk for various policies under data generating processes with different $\Gamma^*$. The proposed method approaches the ideal policy value of -1.41 under optimal policy given the true CATE. Average test-set policy risk errors and 95\% confidence intervals over 20 randomly generated datasets are reported. This shows that our method scales well to large-sample, high-dimensional datasets.}
    \begin{tabular}{c|ccc}
        \toprule
        & \multicolumn{3}{c}{Proposed Method} \\
         $\log{\Gamma^*}$ & $\hat{\pi}(\x; \exp(0.5))$ & $\hat{\pi}(\x; \exp(1.0))$ & $\hat{\pi}(\x; \exp(1.5))$ \\
         \midrule
         0.5 & $\mathbf{\text{-}1.40\pm0.01}$ & $\text{-}1.36\pm0.01$ & $\text{-}1.35\pm0.01$ \\
         1.0 & $\text{-}1.32\pm0.02$ & $\mathbf{\text{-}1.40\pm0.01}$ & $\text{-}1.36\pm0.01$ \\
         1.5 & $\text{-}1.98\pm0.02$ & $\text{-}1.30\pm0.02$ & $\mathbf{\text{-}1.38\pm0.01}$ \\ 
         \bottomrule
    \end{tabular}
    \label{tab:HC-MNIST}
    \vspace{-4mm}
\end{table}

\subsection{IHDP Hidden Confounding}

\begin{figure}[ht]
    \centering
    \begin{subfigure}{0.46\textwidth}
        \centering
        \includegraphics[width=\linewidth]{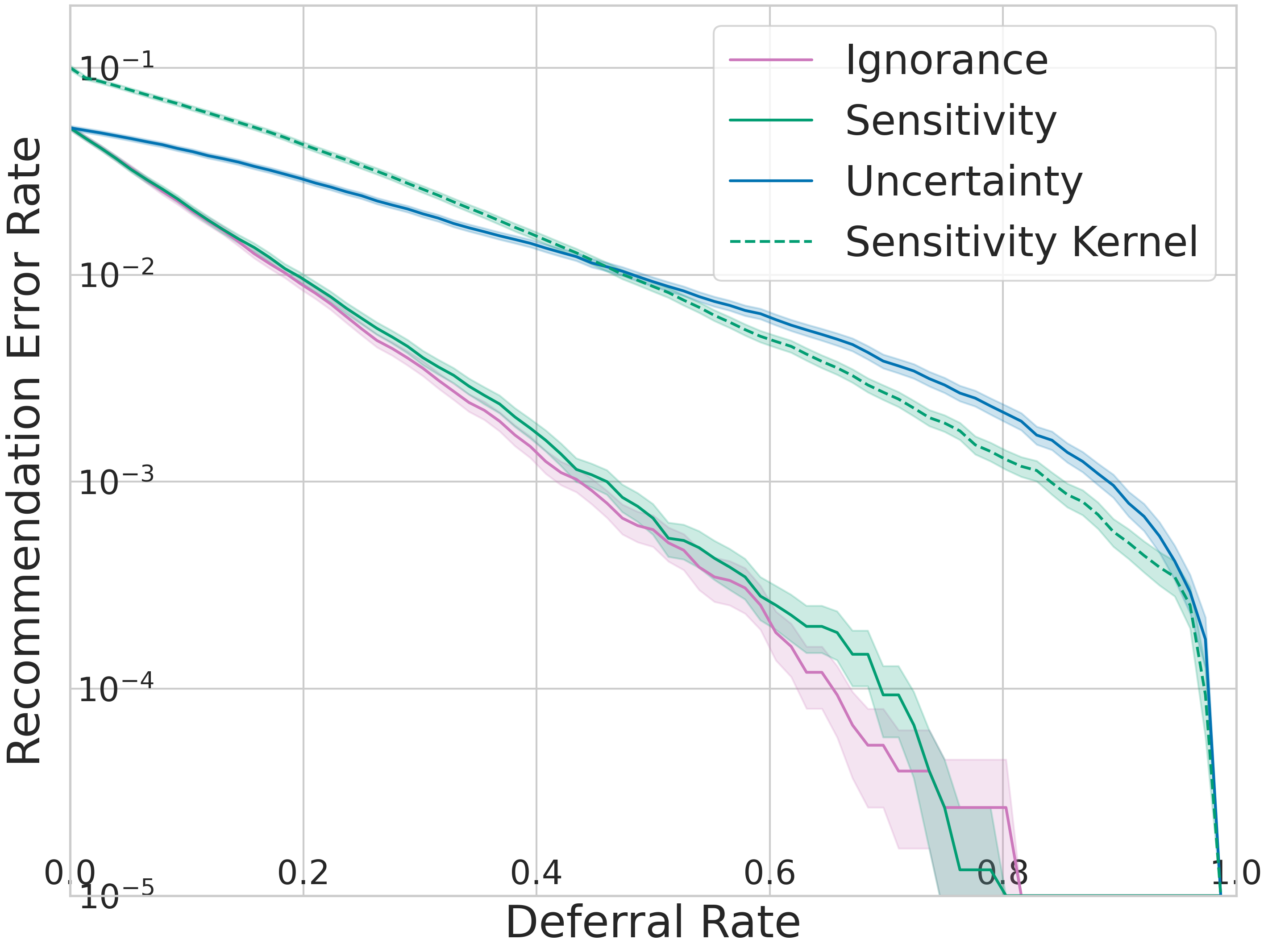}
        \label{fig:context_sim}
        \vspace{-4mm}
    \end{subfigure}
    \caption{\textbf{IHDP Hidden Confounding:} Error rate as we sweep over the percentage of deferred points.
    We propose that recommendations should be deferred when there is ignorance.
    On the x-axis we vary the share of recommendations deferred, simulating various levels of practitioner caution.
    \textit{Ignorance} (ours) accounts for all lack of knowledge. 
    \textit{Uncertainty} \citep{jesson2020identifying} accounts only for insufficient similarity and overlap. 
    \textit{Sensitivity} only accounts for hidden confounding, without accounting for insufficient similarity and overlap; implemented by omitting the variance term in Eq. \eqref{eq:tau_hats}. \textit{Sensitivity Kernel} is the kernel method of \citet{pmlr-v89-kallus19a}, which does not account for other sources of ignorance. 
    Results show that all sources of ignorance are important on IHDP with one hidden confounder.}
    \label{fig:ihdp}
    \vspace{-1em}
\end{figure}

In this section we demonstrate how our uncertainty-aware interval estimator can be used to inform deferral policies for treatment recommendations. 
To this end we use the IHDP dataset \citep{hill2011bayesian} as \citet{jesson2020identifying} show that low overlap and/or similarity are problems for IHDP.
For insufficient context, we induce hidden confounding by hiding covariate $\mathrm{x}_{9}$ during model training and CATE estimation; however, it is still used for the generation of synthetic observed outcomes as per the response surface B described by \citet{hill2011bayesian} \ref{a:datasets_hcihdp}. 

In contrast to the above experiments, treatment $T=1$ is recommended if and only if $\tau(\x) > 0$; we propose a deferral policy that simulates deferral to an expert and withholds a recommendation if the predicted CATE interval intersects 0.
We select $\Gamma_s$ such that the uncertainty aware CATE interval $[\widehat{\underline{\tau}}(\x; \Gamma_s), \widehat{\overline{\tau}}(\x; \Gamma_s)]$ crosses 0.
We then defer predictions with the lowest $\Gamma_s$ value; these are predictions the model is least sure about.
We compare using the same policy for the \citet{pmlr-v89-kallus19a} method, and to the epistemic uncertainty based method proposed by \citet{jesson2020identifying}. 
We report the error rate between recommendations given by $\mathbb{I}(\tau(\x) > 0)$ and $\mathbb{I}(\widehat{\underline{\tau}}(\x) > 0)$ on the remaining recommendations that were not deferred. 

In Figure \ref{fig:ihdp}, we see that the epistemic \emph{uncertainty} policy (blue solid line) has a moderate decrease in error rate as the rate of deferral increases. 
The green solid \emph{sensitivity} line shows that the error rate decreases as we defer recommendations based only on levels of hidden confounding. 
We should see the same behavior for the sensitivity method (green dashed line) proposed by \citet{pmlr-v89-kallus19a}, but it appears to struggle for higher dimensional covariates.
The purple solid \emph{ignorance} line shows that using the uncertainty aware CATE interval further improves results, showing that our method can account for all sources of ignorance discussed.

\section{Conclusion}
In this paper we aim to create a framework for jointly expressing the multiple sources of uncertainty, or ignorance, in individual-level causal inference. 
This includes uncertainty due to finite samples and due to possible violations of the standard causal inference assumptions of overlap and no-hidden confounding, as well as uncertainty due to out-of-distribution data. 
The novel interval estimator we present can scale to large samples and high-dimensional data, and performs well on semi-synthetic, high-dimensional datasets. 
We hope this work leads to further interest in research encompassing the varied possible sources of uncertainty in statistical machine learning models.

\section{Acknowledgements}
We would like to thank Joost van Amersfoort, Jan Brauner, OATML group members, and all anonymous reviewers for sharing their valuable feedback and insights.
Further, we would like to thank Angela Zhou and Tim G. J. Rudner for pointing out inaccuracies that they found in the background and the proof for theorem \ref{th:convergence} after the publication of this work. These have now been corrected.
A.J. would like to thank Lisa, Milad, Joost, Luisa, Lewis, Tim, and Andreas for their friendship and support over a very challenging year. 
U.S. was partially supported by the Israel Science Foundation (grant No. 1950/19).

\bibliography{references}
\bibliographystyle{icml2021}


\newpage
\onecolumn
\appendix

\section{CATE Interval}

\begin{lemma}
    The unbiased estimate of the expected potential outcome under hidden confounding, given in Equation \eqref{eq:unbiased_mu} has the following equivalent characterization:
    \begin{equation}
        \E\left[\Yt \mid \X=\x \right] = \mu_{\ti}(\x) + \frac{ \int (\y - \mu_{\ti}(\x)) w_\ti(\y \mid \x) f(\y \mid \x, \ti) d\y }{\int w_\ti(\y \mid \x) f(\y \mid \x, \ti) d\y}.
    \end{equation}
    \label{lem:double}
\end{lemma}

\begin{proof}
    \begin{subequations}
        \begin{align}
            \E\left[\Yt \mid \X=\x \right] &= \mu_{\ti}(w_\ti; \x) \\
            &= \frac{\int \y w_\ti(\y \mid \x) f_{\ti}(\y \mid \x) d\y}{\int w_\ti(\y \mid \x) f_{\ti}(\y \mid \x) d\y} \\
            &= \int \y\frac{w_\ti(\y \mid \x) f_{\ti}(\y \mid \x)}{\int w_\ti(\y^{\prime} \mid \x) f_{\ti}(\y^{\prime} \mid \x) d\y^{\prime}} d\y \\
            &= \mu_{\ti}(\x) + \int \y\frac{w_\ti(\y \mid \x) f_{\ti}(\y \mid \x)}{\int w_\ti(\y^{\prime} \mid \x) f_{\ti}(\y^{\prime} \mid \x) d\y^{\prime}} d\y - \mu_{\ti}(\x) \\
            &= \mu_{\ti}(\x) + \int \y\frac{w_\ti(\y \mid \x) f_{\ti}(\y \mid \x)}{\int w_\ti(\y^{\prime} \mid \x) f_{\ti}(\y^{\prime} \mid \x) d\y^{\prime}} d\y - \mu_{\ti}(\x) \int \frac{w_\ti(\y \mid \x) f_{\ti}(\y \mid \x)}{\int w_\ti(\y^{\prime} \mid \x) f_{\ti}(\y^{\prime} \mid \x) d\y^{\prime}} d\y \\
            &= \mu_{\ti}(\x) + \int \y\frac{w_\ti(\y \mid \x) f_{\ti}(\y \mid \x)}{\int w_\ti(\y^{\prime} \mid \x) f_{\ti}(\y^{\prime} \mid \x) d\y^{\prime}} d\y - \int \mu_{\ti}(\x) \frac{w_\ti(\y \mid \x) f_{\ti}(\y \mid \x)}{\int w_\ti(\y^{\prime} \mid \x) f_{\ti}(\y^{\prime} \mid \x) d\y^{\prime}} d\y \\
            &= \mu_{\ti}(\x) + \int (\y - \mu_{\ti}(\x)) \frac{w_\ti(\y \mid \x) f_{\ti}(\y \mid \x)}{\int w_\ti(\y^{\prime} \mid \x) f_{\ti}(\y^{\prime} \mid \x) d\y^{\prime}} d\y \\
            &= \mu_{\ti}(\x) + \frac{ \int (\y - \mu_{\ti}(\x)) w_\ti(\y \mid \x) \et(\x) f(\y \mid \x, \ti) d\y }{\int w_\ti(\y \mid \x) \et(\x) f(\y \mid \x, \ti) d\y} \\
            &= \mu_{\ti}(\x) + \frac{ \int (\y - \mu_{\ti}(\x)) w_\ti(\y \mid \x) f(\y \mid \x, \ti) d\y }{\int w_\ti(\y \mid \x) f(\y \mid \x, \ti) d\y}.
        \end{align}
        \label{eq:from_kallus_to_us}
    \end{subequations}
\end{proof}

\begin{lemma}
    The bounds for the conditional expected potential outcomes $\underline{\mu}_{\ti}(\x; \Gamma)$ and $\overline{\mu}_{\ti}(\x; \Gamma)$ defined in equations \eqref{eq:mu_bounds} have the following equivalent characterization:
    \begin{subequations}
        \begin{align*}
            \underline{\mu}_{\ti}(\x; \Gamma) &= \inf_{\y^* \in \mathcal{Y}} \mu_{\ti}(\x) + \frac{ \int_{-\infty}^{\y^*} \rt(\y; \x) f(\y \mid \x, \ti) d\y }{ \alpha_{\ti}^{\prime}(\x; \Gamma) + \mathrm{P}(\Y \leq \y^* \mid \x, \ti) }, \\
            \overline{\mu}_{\ti}(\x; \Gamma) &= \sup_{\y^* \in \mathcal{Y}} \mu_{\ti}(\x) + \frac{ \int_{\y^*}^{\infty} \rt(\y; \x) f(\y \mid \x, \ti) d\y }{ \alpha_{\ti}^{\prime}(\x; \Gamma) + \mathrm{P}(\Y > \y^* \mid \x, \ti) },
        \end{align*}
    \end{subequations}
    where $\rt(\y; \x) = (\y - \mu_{\ti}(\x))$ and $\alpha_{\ti}^{\prime}(\x; \Gamma) = \frac{\alpha_{\ti}(\x; \Gamma)}{\beta_{\ti}(\x; \Gamma) - \alpha_{\ti}(\x; \Gamma)}$. 
    \label{lem:density}
\end{lemma}

\begin{proof}
We prove the result for $\underline{\mu}_{\ti}(\x; \Gamma)$ and the result for $\overline{\mu}_{\ti}(\x; \Gamma)$ can be proved analogously. From \citet{pmlr-v89-kallus19a} Lemma 1,
\begin{equation}
    \begin{split}
        \underline{\mu}_{\ti}(\x) &= \inf_{w_\ti(\y \mid \x) \in [\alpha_{\ti}(\x; \Gamma), \beta_{\ti}(\x; \Gamma)]} \frac{\int \y w_\ti(\y \mid \x) f_{\ti}(\y \mid \x) d\y}{\int w_\ti(\y \mid \x) f_{\ti}(\y \mid \x) d\y} \\
        &= \inf_{u \in \mathcal{U}^{ni}} \frac{ \alpha_{\ti}(\x; \Gamma) \int \y f_{\ti}(\y \mid \x) d\y + (\beta_{\ti}(\x; \Gamma) - \alpha_{\ti}(\x; \Gamma)) \int u(\y) \y f_{\ti}(\y \mid \x) d\y }{ \alpha_{\ti}(\x; \Gamma) \int f_{\ti}(\y \mid \x) d\y + (\beta_{\ti}(\x; \Gamma) - \alpha_{\ti}(\x; \Gamma)) \int u(\y) f_{\ti}(\y \mid \x) d\y }, \\
        &= \underline{\mu}_{\ti}(\x; \Gamma)
    \end{split}
\end{equation}
where $\mathcal{U}^{ni} = \{ u: \mathcal{Y} \xrightarrow[]{} [0, 1] \mid u(\y \text{) is non-increasing} \}$. Therefore, from the equivalence in Equation \eqref{eq:from_kallus_to_us},
\begin{equation*}
    \begin{split}
        \underline{\mu}_{\ti}(\x; \Gamma) &= \!\! \inf_{u \in \mathcal{U}^{ni}} \!\! \mu_{\ti}(\x) + \frac{ \alpha_{\ti}(\x; \Gamma) \int (\y - \mu_{\ti}(\x)) f(\y \mid \x, \ti) d\y + (\beta_{\ti}(\x; \Gamma) - \alpha_{\ti}(\x; \Gamma)) \int u(\y) (\y - \mu_{\ti}(\x)) f(\y \mid \x, \ti) d\y }{ \alpha_{\ti}(\x; \Gamma) \int f(\y \mid \x, \ti) d\y + (\beta_{\ti}(\x; \Gamma) - \alpha_{\ti}(\x; \Gamma)) \int u(\y) f(\y \mid \x, \ti) d\y } \\
        &= \inf_{u \in \mathcal{U}^{ni}} \mu_{\ti}(\x) + \frac{ \alpha_{\ti}(\x; \Gamma) (\mu_{\ti}(\x) - \mu_{\ti}(\x)) + (\beta_{\ti}(\x; \Gamma) - \alpha_{\ti}(\x; \Gamma)) \int u(\y) (\y - \mu_{\ti}(\x)) f(\y \mid \x, \ti) d\y }{ \alpha_{\ti}(\x; \Gamma) \int f(\y \mid \x, \ti) d\y + (\beta_{\ti}(\x; \Gamma) - \alpha_{\ti}(\x; \Gamma)) \int u(\y) f(\y \mid \x, \ti) d\y } \\
        &= \inf_{u \in \mathcal{U}^{ni}} \mu_{\ti}(\x) + \frac{ (\beta_{\ti}(\x; \Gamma) - \alpha_{\ti}(\x; \Gamma)) \int u(\y) (\y - \mu_{\ti}(\x)) f(\y \mid \x, \ti) d\y }{ \alpha_{\ti}(\x; \Gamma) + (\beta_{\ti}(\x; \Gamma) - \alpha_{\ti}(\x; \Gamma)) \int u(\y) f(\y \mid \x, \ti) d\y } \\
        &= \inf_{u \in \mathcal{U}^{ni}} \mu_{\ti}(\x) + \frac{ \int u(\y) (\y - \mu_{\ti}(\x)) f(\y \mid \x, \ti) d\y }{ \alpha_{\ti}^{\prime}(\x; \Gamma) + \int u(\y) f(\y \mid \x, \ti) d\y } \\
        &= \inf_{y^* \in \mathcal{Y}} \mu_{\ti}(\x) + \frac{ \int_{-\infty}^{\y^*} (\y - \mu_{\ti}(\x)) f(\y \mid \x, \ti) d\y }{ \alpha_{\ti}^{\prime}(\x; \Gamma) + \int_{-\infty}^{\y^*} f(\y \mid \x, \ti) d\y } \\
        &= \inf_{y^* \in \mathcal{Y}} \mu_{\ti}(\x) + \frac{ \int_{-\infty}^{\y^*} (\y - \mu_{\ti}(\x)) f(\y \mid \x, \ti) d\y }{ \alpha_{\ti}^{\prime}(\x; \Gamma) + \mathrm{P}(\Y \leq \y^* \mid \x, \ti) } \\
        &= \inf_{y^* \in \mathcal{Y}} \mu_{\ti}(\x) + \frac{ \int_{-\infty}^{\y^*} \rt(\y; \x) f(\y \mid \x, \ti) d\y }{ \alpha_{\ti}^{\prime}(\x; \Gamma) + \mathrm{P}(\Y \leq \y^* \mid \x, \ti) }.
    \end{split}
    \label{eq:mu_bottom}
\end{equation*}
\end{proof}

\section{CATE Interval Estimator}
\label{sec:estimator_proof}
\renewcommand{\proofname}{Proof for Theorem 1}
\begin{proof}
    Here we prove that $\widehat{\underline{\mu}}_{\ti}(\x; \Gamma, \w) \xrightarrow[]{p} \underline{\mu}_{\ti}(\x; \Gamma)$, from which $\widehat{\overline{\mu}}_{\ti}(\x; \Gamma, \w) \xrightarrow[]{p} \overline{\mu}_{\ti}(\x; \Gamma)$ can be proved analogously. Note that $\xrightarrow[]{p}$ indicates convergence in probability. As a reminder
    \begin{equation}
        \widehat{\underline{\mu}}_{\ti}(\x; \Gamma, \w) = \inf_{\y^* \in \mathcal{Y}} \widehat{\mu}_{\ti}(\x; \w) + \frac{ \int_{-\infty}^{\y^*} (\y - \widehat{\mu}_{\ti}(\x; \w)) f(\y \mid \x, \ti, \w) d\y }{ \alpha_{\ti}^{\prime}(\x; \Gamma, \w) + \int_{-\infty}^{\y^*} f(\y \mid \x, \ti, \w) d\y } ,
    \end{equation}
    and
    \begin{equation}
        \underline{\mu}_{\ti}(\x; \Gamma) = \inf_{\y^* \in \mathcal{Y}} \mu_{\ti}(\x) + \frac{ \int_{-\infty}^{\y^*} (\y - \mu_{\ti}(\x)) f(\y \mid \x, \ti) d\y }{ \alpha_{\ti}^{\prime}(\x; \Gamma) + \mathrm{P}(\Y \leq \y^* \mid \x, \ti) }.
    \end{equation}
    Further, our assumptions are
    \begin{enumerate}
        \item $n \to \infty$, and $\x \in \D$.
        \item $\Y$ is a bounded random variable.
        \item $f(\y \mid \x, \ti, \w)$ converges in measure to $f(\y \mid \x, \ti)$. Specifically, $\lim_{n \to \infty} P(\{ \y \in \mathcal{Y}: \abs{f(\y \mid \x, \ti) - f(\y \mid \x, \ti, \w, \D_n)} \geq \epsilon\}) = 0$, for every $\epsilon \geq 0$, where $\D_n$ is a dataset of size $n$. Convergence in measure is a generalization of convergence in probability. 
        \item $\eh_{\ti}(\x; \w)$ and $\widehat{\mu}_{\ti}(\x; \w)$ are consistent estimators of $\E[\T=\ti \mid \X=\x]$ and $\E[\Y \mid \X=\x, \T=\ti]$.
        \item $\et(\x, \y)$ is bounded away from 0 and 1 uniformly over $\x \in \mathcal{X}$, $\y \in \mathcal{Y}$, and $\ti \in \{0, 1\}$.
    \end{enumerate}
    We need to show that $\lim_{n \to \infty} P(\abs{\widehat{\underline{\mu}}_{\ti}(\x; \Gamma, \w) -  \underline{\mu}_{\ti}(\x; \Gamma)} \geq \epsilon) = 0$, for all $\epsilon > 0$, where the parameters $\w$ are dependent on the size of the dataset $n$.
    First, we define the following quantities:
    \begin{subequations}
        \begin{align*}
            \kappa^{\y^*}_{\y}(\x, \ti; n) &= \int_{-\infty}^{\y^*} (\y - \widehat{\mu}_{\ti}(\x; \w)) f(\y \mid \x, \ti, \w) d\y, 
            \quad I^{\y^*}_{\y}(\x, \ti) = \int_{-\infty}^{\y^*} (\y - \mu_{\ti}(\x)) f(\y \mid \x, \ti) d\y,  \\
            \kappa^{\y^*}(\x, \ti; n) &= \int_{-\infty}^{\y^*} f(\y \mid \x, \ti, \w) d\y, 
            \qquad\qquad\qquad\ \, 
            I^{\y^*}(\x, \ti) = \int_{-\infty}^{\y^*} f(\y \mid \x, \ti) d\y,
        \end{align*}
    \end{subequations}
    so that
    \begin{subequations}
        \begin{align*}
            \widehat{\underline{\mu}}_{\ti}(\x; \Gamma, \w) = \inf_{\y^* \in \mathcal{Y}} \widehat{\mu}_{\ti}(\x; \w) + \frac{\kappa^{\y^*}_{\y}(\x, \ti; n)}{\alpha_{\ti}^{\prime}(\x; \Gamma, \w) + \kappa^{\y^*}(\x, \ti; n)},
            \quad \underline{\mu}_{\ti}(\x; \Gamma) = \inf_{\y^* \in \mathcal{Y}} \mu_{\ti}(\x) + \frac{\quad I^{\y^*}_{\y}(\x, \ti)}{\alpha_{\ti}^{\prime}(\x; \Gamma) + I^{\y^*}(\x, \ti)}.
        \end{align*}
    \end{subequations}
    
    For compactness, we use the following shorthand notation:  $\kappa^{\y^*}_{\y} \equiv \kappa^{\y^*}_{\y}(\x, \ti; n)$, $\kappa^{\y^*} \equiv \kappa^{\y^*}(\x, \ti; n)$, $I^{\y^*}_{\y} \equiv I^{\y^*}_{\y}(\x, \ti)$, $I^{\y^*} \equiv I^{\y^*}(\x, \ti)$, $\alpha_{\w}^{\prime} \equiv \alpha_{\ti}^{\prime}(\x; \Gamma, \w)$, and $\alpha^{\prime} \equiv \alpha_{\ti}^{\prime}(\x; \Gamma)$
    
    Then, we need to express $\abs{\widehat{\underline{\mu}}_{\ti}(\x; \Gamma, \w) -  \underline{\mu}_{\ti}(\x; \Gamma)}$ as a sum of products of the following 4 terms: $\Delta^1(n) = \left| \widehat{\mu}_{\ti}(\x; \w) - \mu_{\ti}(\x) \right|$, $\Delta^2(n) = \left| \alpha^{\prime} - \alpha_{\w}^{\prime} \right|$, $\Delta^3(n) = \sup_{\y^* \in \mathcal{Y}} \left| \kappa^{\y^*}_{\y} - I^{\y^*}_{\y} \right|$, and $\Delta^4(n)= \sup_{\y^* \in \mathcal{Y}} \left| I^{\y^*} - \kappa^{\y^*}  \right|$:
    \begin{subequations}
        \begin{align}
            |
                &\widehat{\underline{\mu}}_{\ti}(\x; \Gamma, \w) -  \underline{\mu}_{\ti}(\x; \Gamma)
            |
            \leq \sup_{\y^* \in \mathcal{Y}} 
            \left| 
                \widehat{\mu}_{\ti}(\x; \w) - \mu_{\ti}(\x) + \frac{\kappa^{\y^*}_{\y}}{\alpha_{\w}^{\prime} + \kappa^{\y^*}} - \frac{I^{\y^*}_{\y}}{\alpha^{\prime} + I^{\y^*}}
            \right|, \label{eq:th_1_a} \\
            &= \abs{\widehat{\mu}_{\ti}(\x; \w) - \mu_{\ti}(\x)} + 
            \sup_{\y^* \in \mathcal{Y}} 
            \left| 
                \frac{\kappa^{\y^*}_{\y}}{\alpha_{\w}^{\prime} + \kappa^{\y^*}} - \frac{I^{\y^*}_{\y}}{\alpha^{\prime} + I^{\y^*}}
            \right|, \label{eq:th_1_b} \\
            &\leq \abs{\widehat{\mu}_{\ti}(\x; \w) - \mu_{\ti}(\x)} + 
            \sup_{\y^* \in \mathcal{Y}} 
            \left\{ 
                \frac{
                    \abs{\kappa^{\y^*}_{\y}} \abs{\alpha^{\prime} - \alpha_{\w}^{\prime}}
                }{
                    \abs{\alpha_{\w}^{\prime} + \kappa^{\y^*}} \abs{\alpha^{\prime} + I^{\y^*}}
                } 
                +
                \frac{
                    \abs{\kappa^{\y^*}_{\y}} \abs{I^{\y^*} - \kappa^{\y^*}}
                }{
                    \abs{\alpha_{\w}^{\prime} + \kappa^{\y^*}} \abs{\alpha^{\prime} + I^{\y^*}}
                } 
                + \frac{
                    \abs{\kappa^{\y^*}_{\y} - I^{\y^*}_{\y}}
                }{
                    \abs{\alpha^{\prime} + I^{\y^*}}
                }
            \right\}, \label{eq:th_1_c} \\
            &= \abs{\widehat{\mu}_{\ti}(\x; \w) - \mu_{\ti}(\x)}  
            + \sup_{\y^* \in \mathcal{Y}} 
            \frac{
                \abs{\kappa^{\y^*}_{\y}} \abs{\alpha^{\prime} - \alpha_{\w}^{\prime}}
            }{
                \abs{\alpha_{\w}^{\prime} + \kappa^{\y^*}} \abs{\alpha^{\prime} + I^{\y^*}}
            } 
            + \sup_{\y^* \in \mathcal{Y}}
            \frac{
                \abs{\kappa^{\y^*}_{\y}} \abs{I^{\y^*} - \kappa^{\y^*}}
            }{
                \abs{\alpha_{\w}^{\prime} + \kappa^{\y^*}} \abs{\alpha^{\prime} + I^{\y^*}}
            } 
            + \sup_{\y^* \in \mathcal{Y}}
            \frac{
                \abs{\kappa^{\y^*}_{\y} - I^{\y^*}_{\y}}
            }{
                \abs{\alpha^{\prime} + I^{\y^*}}
            }, \label{eq:th_1_d} \\
            &= \Delta^1(n)  
            + \Delta^2(n) \sup_{\y^* \in \mathcal{Y}} 
            \frac{
                \abs{\kappa^{\y^*}_{\y}}
            }{
                \abs{\alpha_{\w}^{\prime} + \kappa^{\y^*}} \abs{\alpha^{\prime} + I^{\y^*}}
            } 
            + \Delta^4(n) \sup_{\y^* \in \mathcal{Y}}
            \frac{
                \abs{\kappa^{\y^*}_{\y}}
            }{
                \abs{\alpha_{\w}^{\prime} + \kappa^{\y^*}} \abs{\alpha^{\prime} + I^{\y^*}}
            } 
            + \Delta^3(n) \sup_{\y^* \in \mathcal{Y}}
            \frac{
                1
            }{
                \abs{\alpha^{\prime} + I^{\y^*}}
            }. \label{eq:th_1_e}
        \end{align}
    \end{subequations}
    Line \eqref{eq:th_1_a} by Lemma 3 in \citet{pmlr-v89-kallus19a}. Lines \eqref{eq:th_1_a} - \eqref{eq:th_1_b} by Lemma \ref{lem:sup_add}. Lines \eqref{eq:th_1_b} - \eqref{eq:th_1_c} by Lemma \ref{lem:abs_diff} below. Lines \eqref{eq:th_1_c} - \eqref{eq:th_1_d} by Lemma \ref{lem:sup_add}. Lines \eqref{eq:th_1_d} - \eqref{eq:th_1_e} by Lemma \ref{lem:sup_mult}.
    
    So, we now need only prove that $\Delta^1(n) \xrightarrow[]{p} 0$,  $\Delta^2(n) \xrightarrow[]{p} 0$, $\Delta^3(n) \xrightarrow[]{p} 0$, and $\Delta^4(n) \xrightarrow[]{p} 0$, when $n \to \infty$. 
    Note that both $\Delta^1(n) \xrightarrow[]{p} 0$ and  $\Delta^2(n) \xrightarrow[]{p} 0$ are covered by Assumption 4 of Theorem \ref{th:convergence}; namely, \emph{$\eh_{\ti}(\x; \w)$ and $\widehat{\mu}_{\ti}(\x; \w)$ are consistent estimators of $\E[\T=\ti \mid \X=\x]$ and $\E[\Y \mid \X=\x, \T=\ti]$}. 
    
    First, we prove that $\Delta^4(n) \xrightarrow[]{p} 0$.
    
    \textbf{Prove that}  $\sup_{\y^* \in \mathcal{Y}} \left| I^{\y^*} - \kappa^{\y^*} \right| \xrightarrow[]{p} 0$
    \begin{subequations}
        \begin{align*}
            \sup_{\y^* \in \mathcal{Y}} \left| I^{\y^*} - \kappa^{\y^*} \right| 
            &= \sup_{\y^* \in \mathcal{Y}} 
            \left| 
                \int_{-\infty}^{\y^*} f(\y \mid \x, \ti, \w) d\y - \int_{-\infty}^{\y^*} f(\y \mid \x, \ti) d\y 
            \right| \\
            &= \sup_{\y^* \in \mathcal{Y}}
            \left| 
                \textcolor{green}{P(\y \leq \y^* \mid \x, \ti, \w) - P(\y \leq \y^* \mid \x, \ti)} 
            \right| \\
        \end{align*}
    \end{subequations}
    Convergence in probability implies convergence in distribution ($\lim_{n \to \infty} P_n(\X \leq \x) = P(\X \leq \x)$), so by Assumption 3 in Theorem \ref{th:convergence}
    \begin{subequations}
        \begin{align*}
            \lim_{n \to \infty}P\left(
                \left|
                    \sup_{\y^* \in \mathcal{Y}} \left| I^{\y^*} - \kappa^{\y^*} \right|
                \right| \geq \epsilon
            \right)
            &= \lim_{n \to \infty}P\left(
                \sup_{\y^* \in \mathcal{Y}}
                \left| 
                    \textcolor{green}{P_{\w}(\y \leq \y^* \mid \x, \ti) - P(\y \leq \y^* \mid \x, \ti)} 
                \right|  \geq \epsilon
            \right)\\
            &= P\left(
                \sup_{\y^* \in \mathcal{Y}}
                \left| 
                    P(\y \leq \y^* \mid \x, \ti) - P(\y \leq \y^* \mid \x, \ti) 
                \right|  \geq \epsilon
            \right)\\
            &= P\left(
                \sup_{\y^* \in \mathcal{Y}}
                \left| 
                    0 
                \right|  \geq \epsilon
            \right)\\
            &= P\left(
                0   \geq \epsilon
            \right)\\
            &= 0
        \end{align*}
    \end{subequations}

    Finally, we prove $\Delta^3(n) \xrightarrow[]{p} 0$.
    
    \textbf{Prove that}  $\sup_{\y^* \in \mathcal{Y}} \left| \kappa^{\y^*}_{\y} - I^{\y^*}_{\y} \right| \xrightarrow[]{p} 0$
    \begin{equation*}
        \begin{split}
            \sup_{\y^* \in \mathcal{Y}} & \abs{ \kappa^{\y^*}_{\y} - I^{\y^*}_{\y} } 
            = \sup_{\y^* \in \mathcal{Y}} 
            \left| 
                \int_{-\infty}^{\y^*} \!\!\!\ (\y - \widehat{\mu}_{\ti}(\x; \w)) f(\y \mid \x, \ti, \w) d\y - \int_{-\infty}^{\y^*} \!\!\! (\y - \mu_{\ti}(\x)) f(\y \mid \x, \ti) d\y 
            \right|, \\
            &= \sup_{\y^* \in \mathcal{Y}}
            \left|
                \textcolor{blue}{\int_{-\infty}^{\y^*} \!\!\!\! \y f(\y \mid \x, \ti, \w) d\y  - \int_{-\infty}^{\y^*} \!\!\!\! \y f(\y \mid \x, \ti) d\y} 
                + \mu_{\ti}(\x) \int_{-\infty}^{\y^*} \!\!\!\! f(\y \mid \x, \ti) d\y
                - \widehat{\mu}_{\ti}(\x; \w) \int_{-\infty}^{\y^*} \!\!\!\! f(\y \mid \x, \ti, \w) d\y
            \right|, \\
            &= \sup_{\y^* \in \mathcal{Y}}
            \left|
                \textcolor{blue}{\int_{-\infty}^{\y^*} \!\!\!\! \y f_{\w}(\y) d\y  - \int_{-\infty}^{\y^*} \!\!\!\! \y f(\y) d\y}  
                + \mu_{\ti}(\x) \int_{-\infty}^{\y^*} \!\!\!\! f(\y) d\y
                - \widehat{\mu}_{\ti}(\x; \w) \int_{-\infty}^{\y^*} \!\!\!\! f_{\w}(\y) d\y
            \right|, \\
            &= \sup_{\y^* \in \mathcal{Y}}
            \left|
                \textcolor{blue}{\int_{-\infty}^{\y^*} \!\!\!\! \y f_{\w}(\y) d\y  - \int_{-\infty}^{\y^*} \!\!\!\! \y f(\y) d\y}  
                + (\textcolor{purple}{\mu_{\ti}(\x) - \widehat{\mu}_{\ti}(\x; \w)} + \widehat{\mu}_{\ti}(\x; \w)) \int_{-\infty}^{\y^*} \!\!\!\! f(\y) d\y
                - \widehat{\mu}_{\ti}(\x; \w) \textcolor{green}{\int_{-\infty}^{\y^*} \!\!\!\! (f_{\w}(\y) - f(\y)} + f(\y)) \textcolor{green}{d\y}
            \right|, \\
            &= \sup_{\y^* \in \mathcal{Y}}
            \left|
                \textcolor{blue}{\int_{-\infty}^{\y^*} \!\!\!\! \y f_{\w}(\y) d\y  - \int_{-\infty}^{\y^*} \!\!\!\! \y f(\y) d\y}  
                + (\textcolor{purple}{\mu_{\ti}(\x) - \widehat{\mu}_{\ti}(\x; \w)}) \int_{-\infty}^{\y^*} \!\!\!\! f(\y) d\y
                - \widehat{\mu}_{\ti}(\x; \w) \textcolor{green}{\int_{-\infty}^{\y^*} \!\!\!\! \left( f_{\w}(\y) - f(\y) \right) d\y}
            \right|, \\
            &= \sup_{\y^* \in \mathcal{Y}}
            \left|
                \textcolor{blue}{\int_{-\infty}^{\y^*} \!\!\!\! \y f_{\w}(\y) d\y  - \int_{-\infty}^{\y^*} \!\!\!\! \y f(\y) d\y} 
                + (\textcolor{purple}{\mu_{\ti}(\x) - \widehat{\mu}_{\ti}(\x; \w)}) \int_{-\infty}^{\y^*} \!\!\!\! f(\y) d\y
                - \widehat{\mu}_{\ti}(\x; \w)  \left( \textcolor{green}{P_{\w}(\y \leq \y^* \mid \x, \ti) - P(\y \leq \y^* \mid \x, \ti)} \right)
            \right|.
        \end{split}
    \end{equation*}
    As a first step, we can use the result for $\Delta^4(n)$ to remove the \textcolor{green}{green term} from the supremum and now we need to show that
    \begin{equation*}
        \lim_{n \to \infty}P\left(
            \sup_{\y^* \in \mathcal{Y}}
            \left| 
                \textcolor{blue}{\int_{-\infty}^{\y^*} \y f_{\w}(\y) d\y  - \int_{-\infty}^{\y^*} \y f(\y) d\y}  
                + (\textcolor{purple}{\mu_{\ti}(\x) - \widehat{\mu}_{\ti}(\x; \w)}) \int_{-\infty}^{\y^*} f(\y) d\y
            \right|  \geq \epsilon.
        \right) = 0
    \end{equation*}
    Next, under assumption 4 of Theorem \ref{th:convergence} we have $\textcolor{purple}{\mu_{\ti}(\x) - \widehat{\mu}_{\ti}(\x; \w)}\xrightarrow[]{p} 0$, and we are left finally to show that
    \begin{equation*}
        \lim_{n \to \infty}P\left(
            \sup_{\y^* \in \mathcal{Y}}
            \left| 
                \textcolor{blue}{\int_{-\infty}^{\y^*} \y f_{\w}(\y) d\y  - \int_{-\infty}^{\y^*} \y f(\y) d\y}  
            \right|  \geq \epsilon.
        \right) = 0
    \end{equation*}
    Assumption 2 of Theorem \ref{th:convergence} states that $\Y$ is a bounded random variable. As such, there exists a $g(\y)$ such that $\abs{\y f_{\w}(\y)} \leq g(\y)$ for all $n$ and $y \in \mathcal{Y}$. Therefore, in conjunction with Assumption 3, by Lebesgue's dominated convergence theorem we have $\lim_{n \to \infty} \int_{-\infty}^{\y^*} \y f_{\w}(\y) d\y = \int_{-\infty}^{\y^*} \y f(\y) d\y $ 
    
    \begin{equation*}
        \begin{split}
            \lim_{n \to \infty}P\left(
                \sup_{\y^* \in \mathcal{Y}}
                \left| 
                    \textcolor{blue}{\int_{-\infty}^{\y^*} \y f_{\w}(\y) d\y  - \int_{-\infty}^{\y^*} \y f(\y) d\y} 
                \right|  \geq \epsilon.
            \right)
            &= 
            P\left(
                \sup_{\y^* \in \mathcal{Y}}
                \left| 
                    \lim_{n \to \infty} \textcolor{blue}{\int_{-\infty}^{\y^*} \y f_{\w}(\y) d\y  - \int_{-\infty}^{\y^*} \y f(\y) d\y} 
                \right|  \geq \epsilon.
            \right) \\
            &= P\left(
                \sup_{\y^* \in \mathcal{Y}}
                \left| 
                    \int_{-\infty}^{\y^*} \y f(\y) d\y  - \int_{-\infty}^{\y^*} \y f(\y) d\y 
                \right|  \geq \epsilon.
            \right) \\
            &= P\left(
                \sup_{\y^* \in \mathcal{Y}}
                \left| 
                    0 
                \right|  \geq \epsilon.
            \right) \\
            &= 0
        \end{split}
    \end{equation*}
    
    Therefore, $\widehat{\underline{\mu}}_{\ti}(\x; \Gamma, \w) \xrightarrow[]{p} \underline{\mu}_{\ti}(\x; \Gamma)$, and $\widehat{\overline{\mu}}_{\ti}(\x; \Gamma, \w) \xrightarrow[]{p} \overline{\mu}_{\ti}(\x; \Gamma)$ can be proved analogously, which concludes our proof that both $\widehat{\underline{\tau}}(\x; \Gamma, \w) \xrightarrow[]{p} \underline{\tau}(\x; \Gamma)$, and $\widehat{\overline{\tau}}(\x; \Gamma, \w) \xrightarrow[]{p} \overline{\tau}(\x; \Gamma)$.
\end{proof}

\begin{lemma}
    Let $\kappa^{\y^*}_{\y}$, $\kappa^{\y^*}$, $I^{\y^*}_{\y}$, $I^{\y^*}$, $\alpha_{\w}^{\prime}$, and $\alpha^{\prime}$ take real values. Further, let $\alpha_{\w}^{\prime} + \kappa^{\y^*} > 0$ and $\alpha^{\prime} + I^{\y^*} > 0$. Then,
    \begin{equation}
        \left| 
            \frac{\kappa^{\y^*}_{\y}}{\alpha_{\w}^{\prime} + \kappa^{\y^*}} - \frac{I^{\y^*}_{\y}}{\alpha^{\prime} + I^{\y^*}}
        \right|
        \leq \frac{\left| \kappa^{\y^*}_{\y} \right| \left| \alpha^{\prime} - \alpha_{\w}^{\prime} \right|}{\left| \alpha_{\w}^{\prime} + \kappa^{\y^*} \right| \left| \alpha^{\prime} + I^{\y^*} \right|} + \frac{\left| \kappa^{\y^*}_{\y} \right| \left| I^{\y^*} - \kappa^{\y^*} \right|}{\left| \alpha_{\w}^{\prime} + \kappa^{\y^*} \right| \left| \alpha^{\prime} + I^{\y^*} \right|} + \frac{\left| \kappa^{\y^*}_{\y} - I^{\y^*}_{\y} \right|}{\left| \alpha^{\prime} + I^{\y^*} \right|}
    \end{equation}
    \renewcommand{\proofname}{Proof}
    \begin{proof}
        \begin{subequations}
            \begin{align}
                \left| 
                    \frac{\kappa^{\y^*}_{\y}}{\alpha_{\w}^{\prime} + \kappa^{\y^*}} - \frac{I^{\y^*}_{\y}}{\alpha^{\prime} + I^{\y^*}}
                \right|
                &\leq  \left| \frac{\kappa^{\y^*}_{\y}}{\alpha_{\w}^{\prime} + \kappa^{\y^*}} - \frac{\kappa^{\y^*}_{\y}}{\alpha^{\prime} + I^{\y^*}} \right| + \left| \frac{\kappa^{\y^*}_{\y}}{\alpha^{\prime} + I^{\y^*}} - \frac{I^{\y^*}_{\y}}{\alpha^{\prime} + I^{\y^*}} \right| \label{eq:abs_diff_a} \\
                &= \left| \frac{\kappa^{\y^*}_{\y}}{\alpha_{\w}^{\prime} + \kappa^{\y^*}} - \frac{\kappa^{\y^*}_{\y}}{\alpha^{\prime} + I^{\y^*}} \right| + \frac{\left| \kappa^{\y^*}_{\y} - I^{\y^*}_{\y} \right|}{\left| \alpha^{\prime} + I^{\y^*} \right|} \label{eq:abs_diff_b} \\
                &= \left| \frac{\kappa^{\y^*}_{\y} (\alpha^{\prime} + I^{\y^*}) - \kappa^{\y^*}_{\y} (\alpha_{\w}^{\prime} + \kappa^{\y^*})}{(\alpha_{\w}^{\prime} + \kappa^{\y^*}) (\alpha^{\prime} + I^{\y^*})} \right| + \frac{\left| \kappa^{\y^*}_{\y} - I^{\y^*}_{\y} \right|}{\left| \alpha^{\prime} + I^{\y^*} \right|} \label{eq:abs_diff_c} \\
                &= \left| \frac{\kappa^{\y^*}_{\y} (\alpha^{\prime} - \alpha_{\w}^{\prime})}{(\alpha_{\w}^{\prime} + \kappa^{\y^*}) (\alpha^{\prime} + I^{\y^*})} + \frac{\kappa^{\y^*}_{\y} (I^{\y^*} - \kappa^{\y^*})}{(\alpha_{\w}^{\prime} + \kappa^{\y^*}) (\alpha^{\prime} + I^{\y^*})} \right| + \frac{\left| \kappa^{\y^*}_{\y} - I^{\y^*}_{\y} \right|}{\left| \alpha^{\prime} + I^{\y^*} \right|} \label{eq:abs_diff_d} \\
                &\leq \left| \frac{\kappa^{\y^*}_{\y} (\alpha^{\prime} - \alpha_{\w}^{\prime})}{(\alpha_{\w}^{\prime} + \kappa^{\y^*}) (\alpha^{\prime} + I^{\y^*})} \right| + \left| \frac{\kappa^{\y^*}_{\y} (I^{\y^*} - \kappa^{\y^*})}{(\alpha_{\w}^{\prime} + \kappa^{\y^*}) (\alpha^{\prime} + I^{\y^*})} \right| + \frac{\left| \kappa^{\y^*}_{\y} - I^{\y^*}_{\y} \right|}{\left| \alpha^{\prime} + I^{\y^*} \right|} \label{eq:abs_diff_e} \\
                &= \frac{\left| \kappa^{\y^*}_{\y} \right| \left| \alpha^{\prime} - \alpha_{\w}^{\prime} \right|}{\left| \alpha_{\w}^{\prime} + \kappa^{\y^*} \right| \left| \alpha^{\prime} + I^{\y^*} \right|} + \frac{\left| \kappa^{\y^*}_{\y} \right| \left| I^{\y^*} - \kappa^{\y^*} \right|}{\left| \alpha_{\w}^{\prime} + \kappa^{\y^*} \right| \left| \alpha^{\prime} + I^{\y^*} \right|} + \frac{\left| \kappa^{\y^*}_{\y} - I^{\y^*}_{\y} \right|}{\left| \alpha^{\prime} + I^{\y^*} \right|} \label{eq:abs_diff_f}
            \end{align}
        \end{subequations}
    
        Line \eqref{eq:abs_diff_a} by the triangle inequality for absolute values: $|a - b| \leq |a - c| + |c - b|$. Lines \eqref{eq:abs_diff_a} - \eqref{eq:abs_diff_b} by the right-distributive property for division and preservation of division property for absolute values: $|\frac{a}{b}| = \frac{|a|}{|b|}$. Lines \eqref{eq:abs_diff_b} - \eqref{eq:abs_diff_c} by cross multiplication. Lines \eqref{eq:abs_diff_c} - \eqref{eq:abs_diff_d} by successive application of the distributive property for multiplication and the right-distributive property for division. Lines \eqref{eq:abs_diff_d} - \eqref{eq:abs_diff_e} by the subadditivity property of absolute values. Lines \eqref{eq:abs_diff_e} - \eqref{eq:abs_diff_f} by successive applications of the multiplicativity ($|ab| = |a||b|$) and preservation of division properties for absolute values.
    \end{proof}
    \label{lem:abs_diff}
\end{lemma}

\begin{lemma}
    For sets $A$ and $B$ $\sup(A + B) = \sup(A) + \sup(B)$ \citep{zakon2004mathematical}.
    \label{lem:sup_add}
\end{lemma}

\begin{lemma}
    If $A$ and $B$ are non-empty sets of positive real numbers then $\sup(AB) = \sup(A)\sup(B)$ \citep{zakon2004mathematical}.
    \label{lem:sup_mult}
\end{lemma}

\section{Datasets}
\label{a:datasets}

\subsection{Simulated Data}
\label{a:datasets_simulated}

The simulated dataset presented by \citet{pmlr-v89-kallus19a} is described by the following structural causal model (SCM):
\begin{subequations}
    \begin{align}
        \ui &\coloneqq N_{\ui}, \\
        \mathrm{x} &\coloneqq N_{\mathrm{x}}, \\
        \ti &\coloneqq N_{\ti}, \\
        \y &\coloneqq (2\ti - 1)\mathrm{x} + (2\ti - 1) - 2 \sin(2(2\ti - 1)\mathrm{x}) - 2 (2\ui - 1) (1 + 0.5\mathrm{x}) + N_{\y},
    \end{align}
\end{subequations}
where $N_{\ui} \sim \text{Bern}(0.5)$, $N_{\mathrm{x}} \sim \text{Unif}[-2, 2]$, $N_{\ui} \indep N_{\mathrm{x}}$, $N_{\ti} \sim \text{Bern}(e(\mathrm{x}, \ui))$, $e(\mathrm{x}, \ui) = \frac{\ui}{\alpha_{\ti}(\mathrm{x}; \Gamma^*)} + \frac{1 - \ui}{\beta_{\ti}(\mathrm{x}; \Gamma^*)}$, $e(\mathrm{x}) = \text{sigmoid}(0.75 \mathrm{x} + 0.5)$, and $N_{\y} \sim \mathcal{N}(0, 1)$. 

Remember that only $\mathrm{x}$, $\ti$, and $\y$ are observed. So the bias induced by hidden confounding at $\mathrm{x}$ is given by
\begin{equation}
    \tilde{\tau}(\mathrm{x}) - \tau(\mathrm{x}) = 2 (2 + \mathrm{x}) \left(P(\ui = 1 \mid \T = 1, \mathrm{X} = \mathrm{x}) - P(\ui = 1 \mid \T = 0, \mathrm{X} = \mathrm{x})\right),
\end{equation}
where $\tilde{\tau}(\mathrm{x})$ is the confounded CATE estimate.

Each random realization of the simulated dataset generates 1000 training examples, 100 validation examples, and 1000 test examples. In the experiments we report results over 50 random realizations. The seeds for the random number generators are $i$, $i + 1$, and $i + 2$; $\{i \in [0, 1, \dots, 49]\}$, for the training, validation, and test sets, respectively. Code is available in file /library/datasets/synthetic.py on github at \url{https://github.com/anndvision/quince}.

\subsection{HC-MNIST}
\label{a:datasets_hcmnist}
HC-MNIST is an extension of the above dataset with high-dimensional covariates $\x$.
Specifically, $\x$ are MNIST digits. 
HC-MNIST is described by the following SCM:
\begin{subequations}
    \begin{align}
        \ui &\coloneqq N_{\ui}, \\
        \x &\coloneqq N_{\x}, \\
        \phi &\coloneqq \left( \text{clip}\left( \frac{\mu_{N_{\x}} - \mu_{\mathrm{c}}}{\sigma_{\mathrm{c}}}; -1.4, 1.4 \right) - \text{Min}_{\mathrm{c}} \right) \frac{\text{Max}_{\mathrm{c}} - \text{Min}_{\mathrm{c}}}{1.4 - \text{-}
        1.4}\\
        \ti &\coloneqq N_{\ti}, \\
        \y &\coloneqq (2\ti - 1)\phi + (2\ti - 1) - 2 \sin(2(2\ti - 1)\phi) - 2 (2\ui - 1) (1 + 0.5\phi) + N_{\y},
    \end{align}
\end{subequations}
where $N_{\ui}$, $N_{\ti}$ (swapping $\mathrm{x}$ for $\phi$), and $N_{\y}$ are as described in Appendix \ref{a:datasets_simulated}. 
$N_{\x}$ is a sample of an MNIST image. The sampled image has a corresponding label $c \in [0, \dots, 9]$. 
$\mu_{N_{\x}}$ is the average intensity of the sampled image. $\mu_{\mathrm{c}}$ and $\sigma_{\mathrm{c}}$ are the mean and standard deviation of the average image intensities over all images with label $\mathrm{c}$ in the MNIST training set. 
In other words, $\mu_{\mathrm{c}} = \E[\mu_{N_{\x}} \mid \mathrm{c}]$ and $\sigma^2_{\mathrm{c}} = \var[\mu_{N_{\x}} \mid \mathrm{c}]$. 
To map the high dimensional images $\x$ onto a one-dimensional manifold $\phi$ with the same domain as $\mathrm{x} \in [-2, 2]$ above, we first clip the standardized average image intensity on the range $(-1.4, 1.4)$. 
Each digit class has its own domain in $\phi$, so there is a linear transformation of the clipped value onto the range $[\text{Min}_{\mathrm{c}}, \text{Max}_{\mathrm{c}}]$. 
Finally, $\text{Min}_{\mathrm{c}} = -2 + \frac{4}{10} \mathrm{c}$, and $\text{Max}_{\mathrm{c}} = -2 + \frac{4}{10}(\mathrm{c} + 1)$.

For each random realization of the dataset, the MNIST training set is split into training ($n=35000$) and validation ($n=15000$) subsets using the scikit-learn function train\_test\_split(). 
The test set is generated using the MNIST test set ($n=10000$). 
The random seeds are $\{i \in [0, 1, \dots, 19]\}$ for the 20 random realizations generated. Code to generate this dataset is available in file /library/datasets/hcmnist.py on github at \url{https://github.com/anndvision/quince}.

\subsection{IHDP Hidden Confounding}
\label{a:datasets_hcihdp}
The experimental data from the Infant Health and Development Program (IHDP) are used by \citet{hill2011bayesian} to generate simulated outcomes. 
The treatment group reveives ``intensive high-quality child care and home visits from a trained provider." \citet{hill2011bayesian} uses ``measurements on the child--birth weight, head circumference, weeks born preterm, birth order, first born, neonatal health index, sex, twin status--as well as behaviors engaged in during pregnancy--smoked cigarettes, drank alcohol, took drugs--and measurements on the mother at the time she gave birth--age, marital status, educational attainment (did not graduate from high school, graduated from high school, attended some college but did not graduate, graduated from college), whether she worked during pregnancy, whether she received prenatal care--and the site (8 total) in which the family resided at the start of the intervention. 
There are 6 continuous covariates and 19 binary covariates.'' 
\citet{hill2011bayesian} excludes ``a nonrandom portion of the treatment group: all children with nonwhite mothers," in order to simulate an observational study. 
Table \ref{tab:ihdp_covariates} enumerates the included covariates.
There are 139 examples in the treatment group and 608 examples in the control group, for a total of 747 examples.

\begin{table}[ht]
    \centering
    \caption{\textbf{IHDP Covariates} Binary covariates $x_{9}-\mathrm{x}_{18}$ are attributes of the mother. Mother's education level ``College" indicated by covariates $\mathrm{x}_{10}-\mathrm{x}_{12}$ all zero. Site 8 indicated by covariates $\mathrm{x}_{19}-\mathrm{x}_{25}$ all zero. We show the frequency of occurrence for each binary covariate $p(\mathrm{x} = 1)$, as well as the adjusted mutual information $I(\mathrm{x}; \ti)$ between the binary covariate and the treatment variable.}
    \begin{tabular}{ll|llcc}
        \toprule
        \multicolumn{2}{l|}{Continuous} & \multicolumn{4}{l}{Binary}  \\
        Covariate           & Description               & Covariate         & Description                       & $I(\mathrm{x}; \ti)$  & $p(\mathrm{x} = 1)$   \\
        \midrule
        $\mathrm{x}_{1}$    & birthweight               & $\mathrm{x}_{7}$  & child's gender (female=1)         & 0.00                  & 0.51                  \\
        $\mathrm{x}_{2}$    & head circumference        & $\mathrm{x}_{8}$  & is child a twin                   & 0.00                  & 0.09                  \\
        $\mathrm{x}_{3}$    & number of weeks pre-term  & $\mathrm{x}_{9}$  & married when child born           & \textbf{0.02}         & \textbf{0.52}         \\
        $\mathrm{x}_{4}$    & birth order               & $\mathrm{x}_{10}$ & left High School                  & 0.00                  & 0.36                  \\
        $\mathrm{x}_{5}$    &``neo-natal health index"  & $\mathrm{x}_{11}$ & completed High School             & 0.00                  & 0.27                  \\
        $\mathrm{x}_{6}$    & mom's age                 & $\mathrm{x}_{12}$ & some College                      & 0.00                  & 0.22                  \\
                            &                           & $\mathrm{x}_{13}$ & child is first born               & 0.00                  & 0.36                  \\
                            &                           & $\mathrm{x}_{14}$ & smoked cigarettes when pregnant   & \textbf{0.01}         & \textbf{0.48}         \\
                            &                           & $\mathrm{x}_{15}$ & consumed alcohol when pregnant    & 0.00                  & 0.14                  \\
                            &                           & $\mathrm{x}_{16}$ & used drugs when pregnant          & 0.00                  & 0.96                  \\
                            &                           & $\mathrm{x}_{17}$ & worked during pregnancy           & \textbf{0.01}         & \textbf{0.59}         \\
                            &                           & $\mathrm{x}_{18}$ & received any prenatal care        & \textbf{0.01}         & 0.96                  \\
                            &                           & $\mathrm{x}_{19}$ & site 1                            & 0.00                  & 0.14                  \\
                            &                           & $\mathrm{x}_{20}$ & site 2                            & \textbf{0.01}         & 0.14                  \\
                            &                           & $\mathrm{x}_{21}$ & site 3                            & 0.00                  & 0.16                  \\
                            &                           & $\mathrm{x}_{22}$ & site 4                            & \textbf{0.01}         & 0.08                  \\
                            &                           & $\mathrm{x}_{23}$ & site 5                            & \textbf{0.02}         & 0.07                  \\
                            &                           & $\mathrm{x}_{24}$ & site 6                            & \textbf{0.01}         & 0.13                  \\
                            &                           & $\mathrm{x}_{25}$ & site 7                            & \textbf{0.02}         & 0.16                  \\
        \bottomrule
    \end{tabular}
    \label{tab:ihdp_covariates}
\end{table}

Response surface B, designed by \citet{hill2011bayesian}, is described by the following SCM:
\begin{subequations}
    \begin{align}
        \x &\coloneqq N_{\x}, \\
        \ti &\coloneqq N_{\ti}, \\
        \y &\coloneqq (\ti - 1) \left( \exp(\beta_{\x}(\x + \mathbf{w})) + N_{\Yzero} \right) + \ti \left( \beta_{\x}\x - \omega^s + N_{\Yone}) \right),
    \end{align}
\end{subequations}
where $(N_{\x}, N_{\ti}) \sim p_{\D}(\{\mathrm{x}_{1}, \dots \mathrm{x}_{25} \}, \ti)$, $N_{\Yzero} \sim \mathcal{N}(0, 1)$, and $N_{\Yone} \sim \mathcal{N}(0, 1)$. The coefficients $\beta_{\x}$ are a vector of randomly sampled values $( 0.0, 0.1, 0.2, 0.3, 0.4)$ with probabilities $(0.6, 0.1, 0.1, 0.1, 0.1)$. \citet{hill2011bayesian} describes $\omega^s$ as follows: ``For the $s$th simulation, [$\w^s$] is chosen in the overlap setting, where we estimate the effect of the treatment on the treated [(CATT)], such that CATT equals 4; similarly it was chosen in the incomplete setting, where we estimate the effect of the treatment on the controls [(CATC)], so that CATC equals 4." An offset vector $\mathbf{w}$, equal in dimension to $\x$, with every value set to 0.5, is added to $\x$.

To induce hidden confounding, we need to select a variable $\ui$ that is associated with the treatment that will be hidden from the CATE interval estimator, and design a response surface where the outcome will always be affected by $\ui$. In Table \ref{tab:ihdp_covariates}, we list 3 potential candidates for $\ui$: $\mathrm{x}_{9}$, $\mathrm{x}_{14}$, and $\mathrm{x}_{17}$. Each of these variable have a non-negligible association with the treatment, as indicated by the adjusted mutual information score $I(\mathrm{x}; \ti)$, and have a frequency of taking the value 1 at around 0.5 (increasing the chances that we will have both positive and negative examples in each of the training, validation, and testing splits). Here we select $\mathrm{x}_{9}$ and define the following SCM:
\begin{subequations}
    \begin{align}
        \ui &\coloneqq N_{\ui}, \\
        \x &\coloneqq N_{\x}, \\
        \ti &\coloneqq N_{\ti}, \\
        \y &\coloneqq (\ti - 1) ( \exp(\beta_{\x}(\x + \mathbf{w}) + \beta_{\ui}(\ui + 0.5)) + N_{\Yzero}) + \ti (\beta_{\x}\x + \beta_{\ui}\ui - \omega^s + N_{\Yone})),
    \end{align}
\end{subequations}
where $(N_{\ui}, N_{\x}, N_{\ti}) \sim p_{\D}(\mathrm{x}_{9}, \{\mathrm{x}_{1}, \dots, \mathrm{x}_{8}, \mathrm{x}_{10}, \dots, \mathrm{x}_{25} \}, \ti)$, $N_{\Yzero} \sim \mathcal{N}(0, 1)$, and $N_{\Yone} \sim \mathcal{N}(0, 1)$.
The coefficient $\beta_{\ui}$ is randomly sampled from $( 0.1, 0.2, 0.3, 0.4, 0.5)$ with probabilities $(0.2, 0.2, 0.2, 0.2, 0.2)$. The remaining parameters--$\beta_{\x}$, $\omega^s$, and $\w$--are given as above, taking into account $\ui$.

For each random realization of the dataset, the IHDP data is split into training ($n=470$), validation ($n=202$) and test ($n=75$) subsets using the scikit-learn function train\_test\_split().
The random seeds for both splitting and outcome generation are $\{i \in [0, 1, \dots, 999]\}$ for the 1000 realizations generated. 
Code to generate this dataset is available in file /library/datasets/ihdp.py on github at \url{https://github.com/anndvision/quince}.

\section{Implementation Details}
\label{a:implementation_details}

Experiments for the Simulated and IHDP datasets were run using a single NVIDIA GeForce GTX 1080 ti, an Intel(R) Core(TM) i7-8700K, on a desktop computer with 16GB of RAM. 
Experiments for the HCMNIST dataset were run using 4 NVIDIA GeForce RTX 2080 ti GPUs, an Intel(R) Core(TM) i9-9900K, on a server with 64GB of RAM. 
Code is written in python. Packages used include PyTorch \citep{paszke2019pytorch}, scikit-learn \citep{scikit-learn}, Ray \cite{moritz2018ray}, NumPy, SciPy, and Matplotlib. 
We use ray tune \citep{liaw2018tune} with the hyperopt \cite{pmlr-v28-bergstra13} search algorithm to optimize our network hyper-parameters. 
The hyper-parameters we consider are accounted for in Table \ref{tab:search_space}. The hyper-paramter optimization objective for each dataset is the expected batch-wise log-likelihood of the validation data for a single dataset realization with random seed 1331. 

Each experiment is replicated using the training, validation, and testing datasets described in the previous section.

Code to replicate these experiments is available at \url{https://github.com/anndvision/quince}.

\begin{table}[ht]
    \centering
    \begin{tabular}{ll}
        \toprule
        Hyper-parameter & Search Space \\
        \midrule
        hidden units        & [50, 100, 200, 400] \\
        network depth       & [1, 2, 3, 4, 5] \\
        negative slope      & [ReLU, 0.1, 0.2, 0.3, 0.4, 0.5, ELU] \\
        dropout rate        & [0.00, 0.10, 0.15, 0.20, 0.25, 0.50] \\
        spectral norm       & [None, 0.95, 1.0, 2.5, 3.0, 6.0] \\
        batch size          & [16, 32, 64, 100, 200] \\
        learning rate       & [2e-4, 5e-4, 1e-3] \\
        \bottomrule
    \end{tabular}
    \caption{Hyper-parameter search space}
    \label{tab:search_space}
\end{table}

\begin{table}[ht]
    \centering
    \begin{tabular}{lccc}
        \toprule
        Hyper-parameter & Simulated & HCMNIST & IHDP \\
        \midrule
        hidden units        & 200   & 200   & 200 \\
        network depth       & 4     & 2     & 4 \\
        negative slope      & ReLU  & ReLU  & LeakyReLU 0.3 \\
        dropout rate        & 0.10  & 0.15  & 0.5 \\
        spectral norm       & 6.0  & 3.0  & 6.0 \\
        batch size          & 32    & 200   & 200 \\
        learning rate       & 1e-3  & 5e-4  & 5e-4 \\
        \bottomrule
    \end{tabular}
    \caption{Final hyper-parameters for each dataset}
    \label{tab:hparams}
\end{table}

\subsection{Simulated Data}
\label{a:implementation_simulated}

As a reminder, we need parametric models for the distribution over outcomes $p(\Y \mid \x, \ti, \w)$ and the nominal propensity $\ehw(\x)$.
For $p(\Y \mid \x, \ti, \w)$, we use a 4 hidden layer mixture density network (MDN) \citep{bishop1994mixture} with 5 mixture components.
The 1D treatment variable $\ti$ and 1D covariate $\x$ are concatenated to make a 2D network input.
Each hidden layer is comprised of a 100 neuron linear transformation, followed by a ReLU activation function. 
The MDN parameters are inferred using a linear layer to predict the 5 means, a linear layer followed by a softplus activation to predict the square root of the 5 variances, and a linear layer to predict the logits of the 5 mixture components. 
We use the pytorch MixtureSameFamily distribution, with mixture\_distribution=Categorical(.), and component\_distribution=Normal(.) \cite{paszke2019pytorch}.
The objective function optimized is the negative log likelihood for the label $\y$ of the above distribution with the parameters predicted from ($\mathrm{x}, \ti$).
Dropout is applied to the inputs of each layer after the input layer with a rate of 0.1.
Spectral normalization is applied to the weights of the networks with value 6.0.
For $\ehw(\x)$, we use a 4 hidden layer neural network with Bernoulli likelihood.
Each hidden layer is comprised of a 200 neuron linear transformation, followed by a ReLU activation function.
Spectral normalization is applied to the weights of the networks with value 6.0.
The objective function optimized is the negative log likelihood for the observed treatment $\ti$ of the Bernoulli distribution with the logits predicted from $\mathrm{x}$.
For both models, We use Adam optimization with default pytorch parameters \citep{kingma2017adam}.
We use a batch size of 32.
We use early stopping based on the objective function value on the validation set with a patience of 20 epochs and train for a maximum of 500 epochs.
We train an ensemble of 10 models as an estimation of Bayesian model averaging.
At test time, we do 10 MC samples, corresponding to a forward pass of each model in the ensemble for $\w$ and 100 MC samples for $\y$, for each model under $\ti=0$ and $\ti=1$.

\begin{figure}[ht]
    \centering
    \begin{subfigure}{0.5\textwidth}
        \centering
        \includegraphics[width=\linewidth]{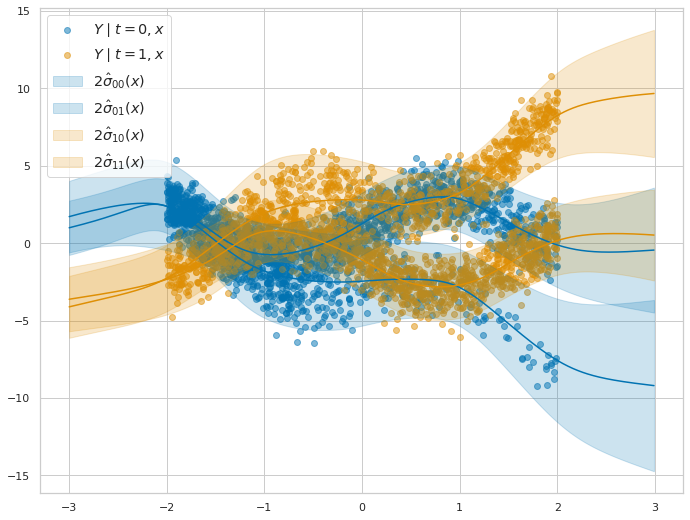}
        \vspace{-4mm}
    \end{subfigure}
    \caption{\textbf{Hidden confounding induces a multi-modal distribution in $\y$ at $\mathrm{x}$}}
    \label{fig:multi_modal}
\end{figure}

\paragraph{Hyper-parameter selection}
The hyper parameter search space is given in Table \ref{tab:search_space} and a summary of the final hyper-parameters used are given in Table \ref{tab:hparams} under the column Simulated.
Because the hidden confounding is a binary variable, it induces a bi-modal distribution in $\y$ at $\x$, as shown in Figure \ref{fig:multi_modal}. In practice, we would not know the form of the distribution of $\y$ at $\x$. To this end we select 5 mixture components for the MDN to show that we can over estimate the true modality, and still obtain sensible results. Alternatively, the validation set could be used to find the number of components that minimizes negative log likelihood of the data. The number of MC samples are chosen based on the stability of network predictions, i.e. we increase the number of MC samples until the variances with respect to $\w$ or $\y$ no longer change significantly.

\subsection{HC-MNIST}
\label{a:implementation_hcmnist}

For $p(\Y \mid \x, \ti, \w)$, we use a ResNet CNN feature extractor with 2 residual blocks, followed by a 2 hidden layer MDN with 5 mixture components.
The 1D treatment variable $\ti$ and ResNet output are concatenated to make a 49 dimensional MDN input.
Each hidden layer in the MDN is comprised of a 200 neuron linear transformation, followed by a ReLU activation function. 
The MDN parameters are inferred using in the same manner as for the Simulated data above.
The objective function optimized is the negative log likelihood for the label $\y$ of the above distribution with the parameters predicted from ($\mathrm{x}, \ti$).
Dropout is applied to the inputs of each layer after the input layer with a rate of 0.15.
Spectral normalization is applied to the weights of the network with value 3.0.
For $\ehw(\x)$, we use a ResNet CNN feature extractor with 2 residual blocks, followed by a 2 hidden layer neural network with Bernoulli likelihood.
Each hidden layer of the Neural Network is comprised of a 200 neuron linear transformation, followed by a ReLU activation function.
Dropout is applied to the inputs of each layer after the input layer with a rate of 0.15.
Spectral normalization is applied to the weights of the network with value 3.0.
The objective function optimized is the negative log likelihood for the observed treatment $\ti$ of the Bernoulli distribution with the logits predicted from $\mathrm{x}$.
For both models, We use Adam optimization with a learning rate of 5e-4 \citep{kingma2017adam}.
We use a batch size of 200.
We use early stopping based on the objective function value on the validation set with a patience of 20 epochs and train for a maximum of 500 epochs.
We train an ensemble of 5 models as an estimation of Bayesian model averaging.
At test time, we do 5 MC samples, corresponding to a forward pass of each model in the ensemble for $\w$ and 100 MC samples for $\y$, for each model under $\ti=0$ and $\ti=1$.

\subsection{IHDP Hidden Confounding}
\label{a:implementation_hcihdp}

For $p(\Y \mid \x, \ti, \w)$, we use a neural network feature extractor with 4 hidden layers, followed by a 2 hidden layer MDN with 5 mixture components.
The 1D treatment variable $\ti$ and feature extractor output are concatenated to make a 201 dimensional MDN input.
Each hidden layer in the feature extractor and MDN is comprised of a 200 neuron linear transformation, followed by an LeakyReLU activation function. 
The MDN parameters are inferred using in the same manner as for the Simulated data above.
The objective function optimized is the negative log likelihood for the label $\y$ of the above distribution with the parameters predicted from ($\mathrm{x}, \ti$).
Dropout is applied to the inputs of each layer after the input layer with a rate of 0.5.
Spectral normalization is applied to the weights of the network with value 6.0.
For $\ehw(\x)$, we use a neural network feature extractor with 3 hidden layers, followed by a 2 hidden layer neural network with Bernoulli likelihood.
Each hidden layer of the Neural Network is comprised of a 200 neuron linear transformation, followed by a ELU activation function.
Dropout is applied to the inputs of each layer after the input layer with a rate of 0.5.
Spectral normalization is applied to the weights of the network with value 6.0.
The objective function optimized is the negative log likelihood for the observed treatment $\ti$ of the Bernoulli distribution with the logits predicted from $\mathrm{x}$.
For both models, We use Adam optimization with a learning rate of 5e-4 \citep{kingma2017adam}.
We use a batch size of 200.
We use early stopping based on the objective function value on the validation set with a patience of 20 epochs and train for a maximum of 500 epochs.
We train an ensemble of 10 models as an estimation of Bayesian model averaging.
At test time, we do 10 MC samples, corresponding to a forward pass of each model in the ensemble for $\w$ and 100 MC samples for $\y$, for each model under $\ti=0$ and $\ti=1$.


\end{document}